\documentclass{article}

\usepackage{microtype}
\usepackage{graphicx}
\usepackage{subcaption}
\usepackage{booktabs} 

\usepackage{hyperref}

\usepackage[accepted]{icml2026}

\usepackage{amsmath}
\usepackage{amssymb}
\usepackage{mathtools}
\usepackage{amsthm} 

\usepackage{multirow} 

\newcommand{\RETURN}{\textbf{Return }} 
\usepackage{tikz}
\usetikzlibrary{positioning, arrows.meta, shapes.geometric, calc}
\usepackage{caption}
\usepackage{arydshln} 
 
\usepackage[capitalize,noabbrev]{cleveref}

\theoremstyle{plain}
\newtheorem{theorem}{Theorem}[section]
\newtheorem{proposition}[theorem]{Proposition}

\theoremstyle{definition}
\newtheorem{definition}[theorem]{Definition}

\theoremstyle{remark}

\icmltitlerunning{DiFA: Inference-Time Forward-Process Alignment for Diffusion Models}

\begin{document}

\twocolumn[
\icmltitle{ 
DiFA: Inference-Time Forward-Process Alignment for Diffusion Models 
}



\icmlsetsymbol{equal}{*}

  \begin{icmlauthorlist}
    \icmlauthor{Shigui Li}{scutmath}
    \icmlauthor{Delu Zeng}{scutece}
  \end{icmlauthorlist}

  \icmlaffiliation{scutmath}{School of Mathematics, South China University of Technology, Guangzhou, China;}
  \icmlaffiliation{scutece}{School of Electronic and Information Engineering, South China University of Technology, Guangzhou, China}

  \icmlcorrespondingauthor{Delu Zeng}{dlzeng@scut.edu.cn}

  \icmlkeywords{Diffusion Models, Flow Matching, Training-Free Inference, Forward Alignment, Kalman Filtering}

  \vskip 0.3in
]



\printAffiliationsAndNotice{}  

\begin{abstract} 
The prevailing inference framework for diffusion models formulates generation fundamentally as a problem of numerical integration. This perspective casts the model as an exact estimator, neglecting the inherent statistical uncertainty of the denoising process. In this work, we propose Forward-Process Aligned Diffusion prediction (\textbf{DiFA}), a training-free  framework that reframes inference-time data prediction refinement  as a sequential state estimation problem. Rather than reusing past outputs solely for numerical integration, DiFA treats  iterative data predictions along the reverse trajectory as correlated observations to build a forward-aligned temporal consensus. Inspired by Kalman filtering,   this consensus aggregates historical predictions according to structural consistency and noise-level compatibility.  To counteract the over-smoothing tendency of temporal consensus, we introduce a deviation guidance mechanism to adaptively preserve residual details. Empirically, DiFA yields significant improvements on CIFAR-10 and ImageNet across the evaluated metrics, including FID, IS, and FD-DINOv2, demonstrating that aligning inference with the forward statistical structure substantially improves generative fidelity.  
\end{abstract}

\section{Introduction}
\label{sec:intro}

Diffusion models (DMs) have demonstrated striking success in high-fidelity generative tasks \citep{sohl2015deep, ho2020denoising,song2021score,dhariwal2021diffusion}, spanning from photorealistic image synthesis to complex scientific simulations \cite{rombach2022high, blattmann2023align,podell2024sdxl, batifol2025flux}. They leverage an iterative denoising process to progressively transform unstructured noise into samples from the target data distribution \citep{kingma2021variational,karras2022elucidating,karras2024analyzing}.   
However,   
this iterative process inherently imposes a significant inference burden \cite{lu2022dpm,zhao2023unipc,wimbauer2024cache, luo2024one}. Unlike one-step generative paradigms (e.g., GANs \citep{NIPS2014_5ca3e9b1}, VAEs \citep{kingma2013auto}), DMs rely on a substantial number of network evaluations (NFEs) to progressively traverse the data manifold and generate
high-fidelity samples \cite{ho2020denoising, farghly2025diffusion}. 
 
The challenge in this diffusion paradigm lies in reducing sampling steps without compromising sample quality \cite{liu2022pseudo,bao2022analyticdpm,zhang2023fast,zheng2023dpm,li2023scire,ma2024sit,sabour2024align,lu2025dpm,jiang2025sada,li2025evodiff,fu2025selfverification, ma2025inference}. Standard samplers, whether deterministic (ODE-based) or stochastic (SDE-based), typically rely on local linear approximations of the score function to predict the next state \cite{lu2022dpm,gonzalez2023seeds,ye2024tfg,abuduweili2025enhancing,tang2025inferencetime}.  
We argue that this linearized perspective is inherently geometrically agnostic.   
In high-curvature regions, such approximations introduce inevitable estimation biases \cite{karras2022elucidating,esser2024scaling,zhang2025antiexposure}. Furthermore, although injected noise is necessary for diversity, it introduces a per-step variance that is severely amplified during discrete, few-step inference \cite{song2021denoising}. 
These errors are not isolated but cumulative \cite{ning2024elucidating,chang2026design,li2026mitigatingcontractivitytrapdiffusion},  
with minor deviations compounding into irreversible prediction drift.  

\begin{figure*}[t]
    \centering 
    \resizebox{\textwidth}{!}{
       \begin{tikzpicture}[node distance=0.8cm and 1.2cm, auto, >=stealth]

            \tikzset{
                block/.style={draw, rectangle, minimum height=3.5em, minimum width=6.5em, align=center, thick, font=\small},
                base_block/.style={block, fill=gray!5, draw=gray!60},
                difa_block/.style={block, fill=blue!8, draw=blue!70!black},
                state_node/.style={draw, ellipse, minimum height=2.6em, minimum width=6.0em, align=center, fill=orange!10, thick, draw=orange!80!black, font=\small},
                fusion/.style={
                    inner sep=0pt, 
                    minimum size=2.6em
                },
                line/.style={draw, -{Latex[length=1.8mm]}, thick},
                main_line/.style={draw, -{Latex[length=2.2mm]}, line width=1.1pt},
                dash_line/.style={draw, -{Latex[length=1.4mm]}, dashed, color=gray!60}
            }

            \node (xt_top) at (1.5,0) {$\boldsymbol{x}_t$};
            \node [base_block, right=of xt_top] (model_top) {Standard Denoiser \\ \scriptsize $D_\theta(\boldsymbol{x}_t, t) \to \hat{\boldsymbol{x}}_0^{(t)}$};
            \node [base_block, right=2.4cm of model_top] (solver_top) {Numerical Solver \\ \scriptsize (Open-Loop)};
            \node [right=of solver_top] (xt_next_top) {$\boldsymbol{x}_{t-1}$};
            
            \path [main_line] (xt_top) -- (model_top);
            \path [main_line] (model_top) -- (solver_top);
            \path [main_line] (solver_top) -- (xt_next_top);
            
            \node [above=0.75cm of model_top.east, text=red!80!black, font=\bfseries\footnotesize] {Statistically Agnostic (Traditional View)};

            \draw [thick, dashed, color=gray!15] (0, -0.9) -- (15, -0.9);

            \node (xt_bot) at (0, -4.2) {$\boldsymbol{x}_t$};
            \node [difa_block, right=0.8cm of xt_bot] (curr_pred) {Current $\hat{\boldsymbol{x}}_0^{(t)}$\\ Prediction };

            \node [fusion, right=1.8cm of curr_pred] (fusion_node) {
                \begin{tikzpicture}[scale=0.42]
 
                    \draw[thick] (-0.1, 1) .. controls (-1, 1) and (-1.3, 0.7) .. (-1.3, 0.2)
                                        .. controls (-1.3, -0.3) and (-0.8, -0.4) .. (-1, -0.6)
                                        .. controls (-1.2, -0.8) and (-0.8, -1.2) .. (-0.1, -1.2) -- cycle;
 
                    \draw[thick] (0.1, 1) .. controls (1, 1) and (1.3, 0.7) .. (1.3, 0.2)
                                        .. controls (1.3, -0.3) and (0.8, -0.4) .. (1, -0.6)
                                        .. controls (1.2, -0.8) and (0.8, -1.2) .. (0.1, -1.2) -- cycle;
 
                    \fill (-0.7, 0.5) circle (0.15); \draw (-0.7, 0.5) -- (-0.4, 0.5) -- (-0.4, 0.2) -- (-0.1, 0.2);
                    \fill (-0.9, -0.1) circle (0.15); \draw (-0.9, -0.1) -- (-0.3, -0.1);
                    \fill (0.7, 0.4) circle (0.15); \draw (0.7, 0.4) -- (0.3, 0.4) -- (0.3, -0.2) -- (0.1, -0.2);
                \end{tikzpicture}
            };
            
            \node [difa_block, right=1.8cm of fusion_node] (boosting) {DiFA-refined \\ Prediction };
            \node [base_block, right=1.0cm of boosting] (solver_bot) {Downstream  \\ Solver};
            \node [right=0.6cm of solver_bot] (xt_next_bot) {$\boldsymbol{x}_{t-1}$};

            \node [state_node] (kalman_state) at (5.2, -2.4) {Temporal Consensus \\ \scriptsize $\hat{\boldsymbol{x}}_0^{\text{cons}}$};
            
            \draw [line, orange!80!black, looseness=4] (kalman_state.110) to [out=140, in=40] (kalman_state.70);
            \node [above=0.25cm of kalman_state, font=\tiny, text=orange!80!black] {Sequential Attention};
 
            \path [main_line] (xt_bot) -- (curr_pred);
            \path [line] (curr_pred) -- (fusion_node) node [pos=0.45, above, font=\tiny] {observation};
            \path [line] (boosting) -- (solver_bot);
            \path [main_line] (solver_bot) -- (xt_next_bot);
            
            \path [dash_line] (curr_pred.north) |- (kalman_state.west) 
                node [pos=0.7, above, font=\tiny, text=gray] {update};
            
            \path [line] (kalman_state.south) |- (fusion_node.north) 
                node [pos=0.2, right, font=\tiny] {consensus $\hat{\boldsymbol{x}}_0^{\text{cons}}$};
             
            \path [line, purple!70] (fusion_node) -- (boosting) 
                node [midway, above, font=\tiny] {deviation $\boldsymbol{r}_t$}; 
            

            \node [font=\tiny, above=0.1cm of boosting, text=blue!70!black, xshift=0.1cm] (formula) {$\hat{\boldsymbol{x}}_0^{\text{DiFA}} = \hat{\boldsymbol{x}}_0^{(t)} +  \omega  \boldsymbol{g}_t$};

            \node [above=1.25cm of solver_bot, text=blue!80!black, xshift=-1.75cm, font=\bfseries\footnotesize] {Forward-Aligned Prediction (DiFA View)};

        \end{tikzpicture}
    }
    
    \vspace{0.2cm}
    \caption{  
 \emph{DiFA Mechanism of Inference-Time Forward-Process Alignment.} 
Traditional samplers (Top) use each denoiser prediction only for the current solver step, discarding temporal redundancy along the trajectory. DiFA (Bottom) constructs a forward-aligned temporal consensus $\hat{\boldsymbol{x}}_0^{\mathrm{cons}}(t)$ from historical  predictions and uses it as an anchor for residual construction. Specifically, $\boldsymbol{r}_t=\hat{\boldsymbol{x}}_0^{(t)}-\hat{\boldsymbol{x}}_0^{\mathrm{cons}}(t)$ is transformed into deviation guidance $\boldsymbol{g}_t=\operatorname{G}(\boldsymbol{r}_t)$, yielding the refined prediction $\hat{\boldsymbol{x}}_0^{\mathrm{DiFA}}(t)=\hat{\boldsymbol{x}}_0^{(t)}+\omega\boldsymbol{g}_t$ without additional NFEs.}
    \label{fig:mechanism}
\end{figure*}
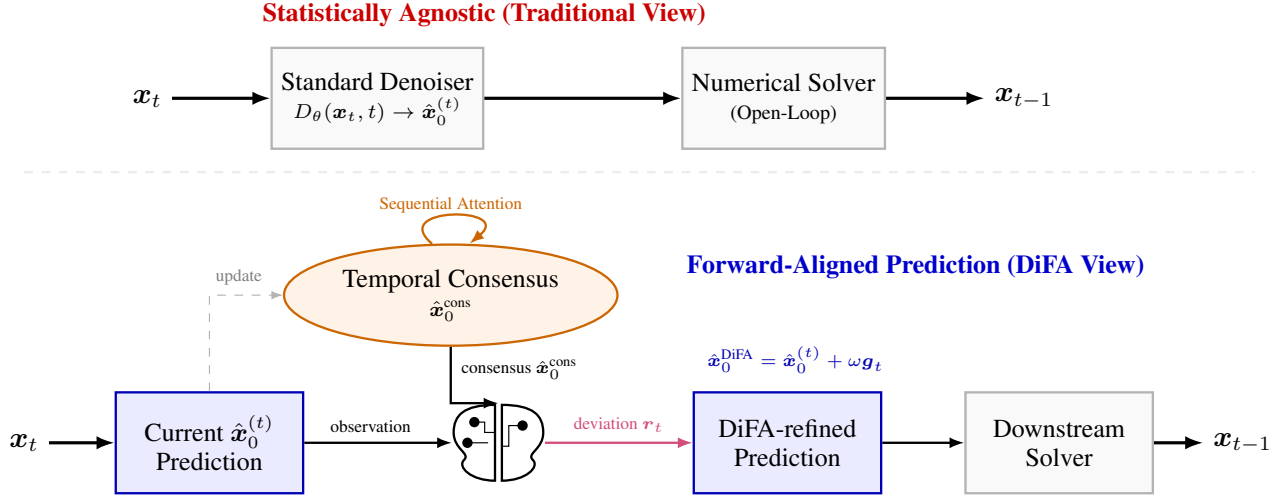
 
To address these challenges, research efforts have largely consolidated into two orthogonal directions, each with distinct limitations. On one hand, distillation techniques attempt to compress the trajectory into fewer steps
\cite{salimans2022progressive,song2023consistency,meng2023distillation,liu2023flow,frans2025one}. While reducing inference latency, they incur substantial retraining overhead and are often confined to task-specific customizations
\cite{luo2023latent,liu2024instaflow,yin2024one,sauer2024adversarial,dong2024continually,shen2025information,sabour2025align}.
This may restrict the model to a narrowed distribution and compromise the zero-shot generalization and diversity of the original teacher model. On the other hand, advanced solvers employ higher-order numerical methods to reduce truncation errors \cite{lu2022dpm}. However, by treating the fixed model as an exact estimator for correcting temporal discretization errors,
they follow each instantaneous prediction as if it were a sufficiently accurate clean-signal estimate, leaving its uncertainty and the temporal redundancy of the prediction sequence unexploited.

Inspired by EDM2 \cite{karras2024guiding}, which suggests DMs can self-correct using just \emph{bad versions} of themselves, we identify the bottleneck not merely time discretization error, but the intrinsic estimation uncertainty of the learned model. Unlike high-order ODE solvers that assume the vector field is ground truth and use auxiliary points solely to refine temporal integration, we argue that the learned score is an imperfect, smoothed approximation of the posterior mean due to the MSE training objective. Consequently, standard samplers treat each prediction as a transient estimation, blindly following a biased tangent while discarding the rich information embedded in the sampling trajectory. \emph{Can we then harvest the latent consensus from the sequence of historical predictions to rectify this estimation drift without any additional network evaluations?} By synthesizing information across the temporal sequence, the model can leverage its own historical redundancy to mitigate linearization biases. The base DM thus holds untapped potential; it merely lacks a mechanism to unleash this intrinsic self-guidance through temporal consensus.  
Guided by this insight, we propose Forward-Process Aligned Diffusion prediction (DiFA),    
a principled inference-time algorithm that rectifies  signal predictions via forward-process alignment. 
DiFA does not replace the numerical solver. Instead, it refines the clean prediction supplied to each solver step. Motivated by  Kalman estimation under an idealized static-anchor model, DiFA implements this principle through forward-aligned sliding-window temporal consensus, where historical predictions are fused according to structural consistency and noise-level compatibility. To avoid over-smoothing, DiFA further constructs a  deviation  guidance direction and modulates its smoothed and high-frequency components.

In summary, our contributions are as follows. First, we identify clean-signal prediction drift during diffusion inference as arising from time-dependent estimation uncertainty, posterior ambiguity, and trajectory mismatch, shifting the focus from solely improving numerical integration to refining the clean-signal estimate provided to downstream solvers. 
Second, we propose DiFA, a solver-compatible inference-time output-alignment framework that exploits temporal redundancy through adaptive temporal consensus, where historical predictions are fused according to structural consistency and noise-level-dependent reliability. 
Third, motivated by Best Linear Unbiased Estimation (BLUE), we introduce a principled approximation to effectively mitigate the aforementioned prediction uncertainty. Extensive evaluations on CIFAR-10 and ImageNet show that DiFA consistently improves existing samplers across FID, IS, and FD-DINOv2 scores without additional network evaluations, demonstrating the effectiveness of correcting clean-signal predictions within the inference pipeline.

\section{Related Work} \label{sec:related}

\paragraph{Accelerated Generation.} 
The substantial computational burden of DMs has catalyzed a broad spectrum of acceleration strategies across distinct operational paradigms. At the architecture level, latent DMs \cite{rombach2022high} improve efficiency by operating in compressed latent spaces. At the training level, researchers accelerate generation by retraining or distilling models to shorten sampling chains. Representative modeling approaches include progressive distillation \cite{salimans2022progressive}, flow matching (FM) \cite{lipman2023flow}, reflow \cite{liu2023flow}, and recently developed one-step or few-step methods, such as consistency models \cite{song2023consistency}, adversarial distillation \cite{sauer2024adversarial}, shortcut models \cite{frans2025one}, and mean flows \cite{geng2025mean}. Alternatively, training-free methods optimize the inference process itself. Following the deterministic formulation \cite{song2021denoising}, a variety of advanced solvers have emerged. Early solvers like PNDM \cite{liu2022pseudo} introduced the use of historical gradients, while the DPM-Solver family \cite{lu2022dpm,lu2025dpm} later proposed specialized exponential integrators for diffusion ODEs. Specialized solvers with lightweight learning have been explored to bridge the gap between both paradigms \cite{tong2025learning}. Moreover, complementary strategies include schedule optimization \cite{sabour2024align}, parallel decoding \cite{shih2023parallel} and caching mechanisms \cite{wimbauer2024cache,ma2024deepcache}. 
While effective in reducing NFEs, these methods prioritize temporal precision, often neglecting the intrinsic estimation variance.

\paragraph{Inference-Time Rectification.}   
A pivotal challenge in diffusion model inference lies in mitigating trajectory instability and statistical discrepancies \citep{lin2024common}. Accordingly, refinement methodologies have progressed from variance-controlled stochastic sampling \cite{ho2020denoising, song2021score, bao2022analyticdpm} to explicit rectification. While advanced solvers like UniPC \cite{zhao2023unipc} emphasize numerical accuracy for efficient integration, recent approaches like Zigzag \cite{bai2025zigzag} and inference-time scaling strategies \cite{ma2025inference} explicitly amend intermediate errors or optimize search paths, often trading computational efficiency for improved generation quality. Recently, to overcome the rigidity of conventional numerical solvers, EVODiff \citep{li2025evodiff} rectifies the inference path via entropy-aware variance optimization.   
Although classical antithetic sampling \cite{ren2019adaptive} has been revisited recently for initial-noise pairing \cite{jia2025antithetic} to improve global diversity, its capacity for dynamic trajectory rectification remains largely unexplored. Existing strategies typically rely on external task-specific constraints \citep{chung2022improving, bansal2024universal, yang2024guidance}, heuristic interventions on internal structures \citep{ho2021classifier, hong2023improving, epstein2023diffusion, liu2023more, Shen_2024_CVPR, ahn2024self}, or temporal rescaling \citep{park2025temporal}. 
However, these methods mainly improve semantic consistency while leaving the inherent estimation variance and high-frequency degradation caused by MSE-based denoising unaddressed. Although EDM2 \cite{karras2024guiding} introduces self-guidance, it still relies on auxiliary degraded models and hand-crafted guidance cues. In contrast, DiFA offers an intrinsic and training-free framework grounded in classical estimation principles. It treats denoising as a filtering process, avoiding reliance on external models or heuristic guidance.

\paragraph{Limitations and Our Contribution.}
Existing inference strategies face a fundamental dilemma: distillation incurs costly retraining and potential quality degradation, whereas standard solvers neglect the estimation variance and detail attenuation of MSE-trained denoisers. We propose Forward-Process Aligned Diffusion prediction (\emph{DiFA}), a training-free framework that exploits temporal redundancy to refine clean-signal predictions. Unlike solvers that use historical model evaluations to refine numerical integration, DiFA constructs a forward-aligned reference in the clean-prediction space before the unchanged solver update, without requiring auxiliary models. Under an idealized independent-view model, the BLUE analysis yields a principled fusion rule for historical predictions, which motivates the practical forward-aligned consensus. Deviation guidance is designed to counteract the detail attenuation introduced by temporal aggregation. By aligning reverse inference with the forward statistical structure, DiFA mitigates prediction drift and better exploits the capacity of pretrained generative models.

\section{Preliminaries}
\label{sec:preliminaries}
\subsection{Diffusion Models and Probability Flow}
Diffusion models (DMs) rely on a constructed  forward process that progressively corrupts data $\boldsymbol{x}_0 \sim q(\boldsymbol{x}_0)$ into Gaussian noise \citep{ho2020denoising}.
The marginal distribution of a noisy sample $\boldsymbol{x}_t$ at time $t$ can generally be expressed as: 
\begin{equation}\label{transitionkernel}
q(\boldsymbol{x}_t|\boldsymbol{x}_0) = \mathcal{N}(\boldsymbol{x}_t; \alpha_t \boldsymbol{x}_0, \sigma_t^2\boldsymbol{I}),
\end{equation}
where $\alpha_t$ and $\sigma_t$ denote the signal and noise scale parameters, respectively. The model learns to reverse this process by predicting the noise component $\boldsymbol{\epsilon}$ via a network $\boldsymbol{\epsilon}_\theta(\boldsymbol{x}_t, t)$, optimized by minimizing the mean squared error (MSE):
\begin{equation}\label{trainingloss}
\operatorname{L} = \mathbb{E}_{t,\boldsymbol{x}_0,\boldsymbol{\epsilon}}\left\|\boldsymbol{\epsilon} - \boldsymbol{\epsilon}_{\boldsymbol {\theta}}(\alpha_t \boldsymbol{x}_0 + \sigma_t\boldsymbol{\epsilon}, t)\right \|^2.
\end{equation} 
While classic DDPMs operate on a discrete schedule, the transition kernel can be generalized to a continuous-time setting \citep{kingma2021variational}. This extension unifies the framework under a stochastic differential equation (SDE) formulation \citep{song2021score}, where the generative process follows the reverse-time dynamics:
\begin{equation}\label{songsde}
    \operatorname{d} \boldsymbol{x} = [f(t)\boldsymbol{x} - g(t)^2 \nabla_{\boldsymbol{x}} \log p_t(\boldsymbol{x})]\operatorname{d} t + g(t)\operatorname{d} \boldsymbol{\bar{w}}, 
\end{equation}
where $\operatorname{d}\boldsymbol{\bar{w}}$ is the standard Wiener process in reverse time. 
Crucially, there exists a deterministic Probability Flow ODE (PF-ODE) that shares the same marginal as the SDE:
\begin{equation}\label{songode}
\frac{\operatorname{d} \boldsymbol{x}}{\operatorname{d} t} = f(t)\boldsymbol{x} - \frac{1}{2} g(t)^2 \nabla_{\boldsymbol{x}} \log p_t(\boldsymbol{x}) .
\end{equation}

\begin{figure}[t]
    \centering  
    \includegraphics[width=0.8\linewidth]{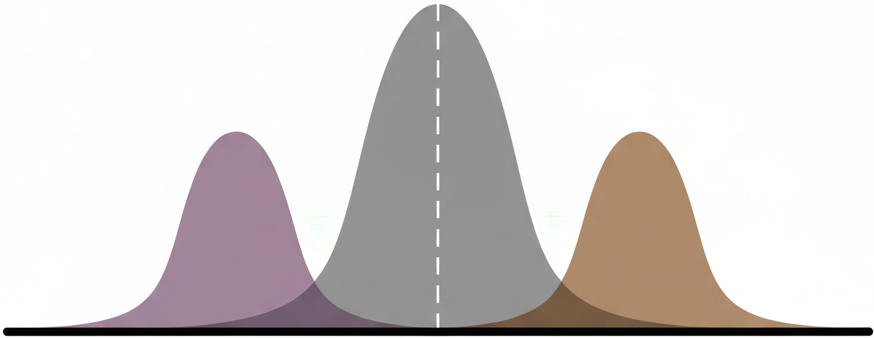}  
    \caption{ 
    Illustrating the mean-field discrepancy in model estimation through spatial shift and magnitude scaling. 
    }
    \label{noiseshift_estimation}
\end{figure}
\subsection{Inference Framework and Linearization Bias}
As the score function $\nabla_{\boldsymbol{x}} \log p_t(\boldsymbol{x})$ is parameterized as $
-\boldsymbol{\epsilon}_\theta(\boldsymbol{x}_t, t) / \sigma_t$ \citep{song2021score,karras2022elucidating} and the prediction relationship of $    \boldsymbol{x}_ {\boldsymbol {\theta}}\left(\boldsymbol{x}_t, t\right)=\frac{\boldsymbol{x}_t-\sigma_t\boldsymbol{\epsilon}_ {\boldsymbol {\theta}}\left(\boldsymbol{x}_t, t\right)}{\alpha_t}$ \cite{kingma2021variational}, the inference ODE is formulated as  
\begin{equation}\label{doded}
\frac{\operatorname{d} \boldsymbol{x} }{\operatorname{~d} t}=\left(f(t)+\frac{g^2(t)}{2 \sigma_t^2} \right) \boldsymbol{x}_t - \alpha_t\frac{g^2(t)}{2 \sigma_t^2} \boldsymbol{x}_ {\boldsymbol {\theta}}\left(\boldsymbol{x}_t, t\right).
\end{equation} 
This framework is generalized by the EDM formulation \citep{karras2022elucidating,karras2024guiding}, which stabilizes probability flow via signal scaling. A pivotal innovation in EDM is the introduction of input preconditioning, which explicitly standardizes the network input into a \emph{signal-centric form}:
\begin{equation}
\label{eq:edm_scaling}
\frac{\boldsymbol{x}_t}{s(t)} = \boldsymbol{x}_0 + \sigma(t)\boldsymbol{\epsilon}.
\end{equation}
By normalizing the signal scale to unity across all time $t$, this formulation reveals that the diffusion trajectory effectively orbits a static anchor $\boldsymbol{x}_0$. This architectural alignment of the input state to $\boldsymbol{x}_0$ resonates with the core intuition of our work, where we seek to align the output predictions to the same statistical invariant. This concept of straightening the generative path is also central to concurrent works on  FMs \cite{lipman2023flow, liu2023flow}. In fact, FM can be seen as a special case of diffusion modeling  \cite{kingma2023understanding}, which typically defines the process as a linear interpolation $\boldsymbol{x}_{t}=(1-t)\boldsymbol{x}_{0}+t\boldsymbol{\epsilon}$.  
While trajectory straightness is widely regarded as a key property for minimizing discretization errors, the perspective of rectified diffusion \cite{wang2025rectified} suggests that pre-defined linearity is insufficient. Instead, the efficacy of rectification stems from  bootstrap retraining, which aligns the model with ODE coupling of a pre-trained teacher. This process facilitates the transport curvature, allowing first-order solvers to achieve one-step generation \cite{frans2025one,geng2025mean}.

In this work, we focus on the general EDM formulation without such expensive retraining. Consequently, our standard pre-trained models do not strictly satisfy the first-order property. To quantify the impact of this, let $\Phi(\boldsymbol{x}, t)$ denote the drift vector field defined by the RHS of Eq. \eqref{doded}. Standard inference numerically integrates this ODE by approximating the trajectory as locally linear within a step $\Delta t$ (e.g., a first-order Euler step $\boldsymbol{x}_{t-1} \approx \boldsymbol{x}_t + \Delta t \cdot \Phi(\boldsymbol{x}_t, t)$). However, due to the lack of retraining, the generative manifold retains significant curvature. This geometric complexity, coupled with the inherent estimation variance of the model $\Phi$, renders the trajectory highly sensitive to local prediction errors. Standard solvers, which blindly follow the instantaneous tangent, accumulate these errors as geometric drift.

\section{Method}
\label{sec:method}

We propose Forward-Process Aligned Diffusion prediction (DiFA), a theory-inspired and principled algorithm for inference-time clean-signal prediction alignment. DiFA is a training-free framework and does not replace existing numerical solvers. Instead, it operates at the interface between the pretrained denoiser and the downstream solver by refining  the data prediction supplied to each solver step.  From the perspective of the forward process of diffusion models, the
central idea is to reframe the inference-time predictions along the reverse
trajectory as temporally related observations of a locally stable
clean-signal anchor. DiFA uses the common-anchor geometry and the
noise-level-dependent reliability structure induced by the forward process
to filter these multi-step predictions before they are passed to the
downstream solver. 

\subsection{Diffusion Sampling as Sequential Estimation}
\label{sec:sequential_estimation} 

During inference, the denoiser produces a sequence of clean-signal predictions,
\begin{equation}\label{eq:observation_model}
\hat{\boldsymbol{x}}_{0}^{(t_i)} = D_{\theta}(\boldsymbol{x}_{t_i},t_i), \qquad i=1,\ldots,N.
\end{equation}
Assuming that reverse inference proceeds in the order $t_N > t_{N-1} > \cdots > t_0$, for a current reverse step $t_i$, let
\begin{equation}
\operatorname{W}_{K}(t_i) \subseteq \{t_j:j>i\}, \qquad |\operatorname{W}_{K}(t_i)|\leq K,
\end{equation}
denote a causal temporal window of recently processed historical steps, and define
\begin{equation}
\operatorname{W}_{K}^{+}(t_i) = \operatorname{W}_{K}(t_i)\cup\{t_i\}.
\end{equation}
Rather than treating the predictions in this local neighborhood as isolated one-step estimates, we view them as temporally correlated observations of a locally stable trajectory-implied clean anchor. For $t_j\in\operatorname{W}_{K}^{+}(t_i)$, we write
\begin{equation}\label{eq:practical_prediction_model}
\hat{\boldsymbol{x}}_{0}^{(t_j)} = \boldsymbol{x}_{0,\star}^{(t_i)} + \boldsymbol{b}_{j\mid i} + \boldsymbol{\xi}_{j\mid i}, ~  t_j\in\operatorname{W}_{K}^{+}(t_i), 
\end{equation}   
where $\boldsymbol{x}_{0,\star}^{(t_i)}$ denotes the trajectory-implied clean anchor, assumed locally stable within this neighborhood, while $\boldsymbol{b}_{j\mid i}$ and $\boldsymbol{\xi}_{j\mid i}$ capture the systematic bias and residual uncertainty with respect to this anchor, respectively. 

Due to distribution shifts in few-step trajectories, predictions from MSE-trained denoisers may exhibit systematic bias and temporal correlation. The reliability ordering used by DiFA is instead induced by the normalized forward process. From Eq.~\eqref{transitionkernel}, the forward process admits the equivalent signal-normalized representation,
\begin{equation}
\bar{\boldsymbol{x}}_t := \frac{\boldsymbol{x}_t}{\alpha_t} = \boldsymbol{x}_0 + \frac{\sigma_t}{\alpha_t}\boldsymbol{\epsilon} = \boldsymbol{x}_0 + \operatorname{SNR}(t)^{-1/2}\boldsymbol{\epsilon},   
\end{equation} 
where $\boldsymbol{\epsilon}\sim\mathcal{N}(\boldsymbol{0},\boldsymbol{I})$. 
Accordingly,
\begin{equation}\label{signal-normalizedrep}
\bar{\boldsymbol{x}}_t \mid \boldsymbol{x}_0 \sim \mathcal{N} \left( \boldsymbol{x}_0, \operatorname{SNR}(t)^{-1}\boldsymbol{I} \right).
\end{equation}
Therefore, in the forward-corruption space, states at different noise levels share the same clean-data anchor ($\boldsymbol{x}_0$), and $\operatorname{SNR}(t)$ exactly characterizes the conditional observation precision of the corresponding forward observation.

This forward statistical geometry motivates DiFA's treatment of the
reverse-trajectory predictions: although learned denoiser errors may be biased and temporally correlated rather than exactly following the forward observation covariance, the prediction sequence can still be organized and filtered according to the common-anchor structure and the noise-level reliability ordering induced by the forward process. While existing numerical solvers widely utilize historical outputs to refine temporal integration, they fail to explicitly and dynamically construct a forward-aligned reference anchor to rectify prediction drift during inference.
 
\subsection{Forward-Process Alignment Principle}
\label{sec:optimal_fusion}  

Based on the forward-induced common-anchor geometry established above, we now formalize the alignment principle underlying DiFA in Algorithm~\ref{alg:difa}.  

\begin{definition}[Forward-Process Alignment]\label{def:fpa}
Let $\{\hat{\boldsymbol{x}}_{0}^{(t_i)}\}_{i=1}^{N}$ denote the data predictions along a reverse trajectory.  
Forward-Process Alignment is the principle that inference-time prediction refinement should be constructed relative to the common-anchor geometry and noise-level reliability ordering induced by the forward noising process. 
\end{definition}

This principle is instantiated by first constructing a temporally stable reference anchor in the clean-prediction space and then using the anchor-relative deviation for current-step refinement. Accordingly, forward-process alignment specifies the reference coordinate and reliability organization for prediction refinement. 
To derive a canonical consensus rule implied by this principle, we
introduce an idealized and  tractable forward-aligned
independent-view observation model. This model preserves the common-anchor geometry and inverse-SNR reliability structure of the normalized forward process, while serving as an analytical proxy rather than an exact model of the learned denoiser errors: 
\begin{equation}
\begin{aligned}
\boldsymbol{y}_{i} &= \boldsymbol{x}_{0,\star} + \boldsymbol{\eta}_{i}, \\
\boldsymbol{\eta}_{i}&\overset{\mathrm{ind}}{\sim} \mathcal{N} \left( \boldsymbol{0}, R_i\boldsymbol{I} \right), ~~R_i = \frac{c}{\operatorname{SNR}(t_i)}, 
\end{aligned}
\label{eq:kalman_ideal_observation_model}
\end{equation} 
where $c>0$ is a common scale constant  and $\boldsymbol{x}_{0,\star}$ denotes the static anchor in this ideal analytical model.  
It is the idealized counterpart of the locally stable trajectory-implied anchor $\boldsymbol{x}_{0,\star}^{(t_i)}$ in Eq.~\eqref{eq:practical_prediction_model}.  
This model provides a canonical analytical basis for deriving the
consensus estimator.

Then, we seek a linear unbiased estimator of $\boldsymbol{x}_{0,\star}$ whose weights minimize the isotropic estimation covariance:
\begin{equation}
\begin{aligned}
\hat{\boldsymbol{x}}_{0,\mathrm{ideal}}^{\star} &= \sum_{i=1}^{n} w_i\boldsymbol{y}_{i}, ~~ \sum_{i=1}^{n}w_i=1, \\
\left\{ w_i^{\star} \right\}_{i=1}^{n} &= \underset{\sum_{i=1}^{n}w_i=1}{\arg\min} \; \sum_{i=1}^{n} w_i^{2}R_i .
\end{aligned}
\label{eq:ideal_blue_objective}
\end{equation}
This yields the precision-weighted estimator
\begin{equation}
\hat{\boldsymbol{x}}_{0,\mathrm{ideal}}^{\star} = \frac{ \sum_{i=1}^{n} \operatorname{SNR}(t_i) \boldsymbol{y}_{i} }{ \sum_{i=1}^{n} \operatorname{SNR}(t_i) },
\label{eq:ideal_snr_fusion}
\end{equation}
where the SNR-dependent weights follow from the inverse-SNR covariance form of the forward-aligned observation.

\begin{proposition}[Variance Reduction under Ideal Forward-Aligned Observations]
\label{prop:variance}
Under the observation model in Eq.~\eqref{eq:kalman_ideal_observation_model},  the fused anchor estimator satisfies
\begin{equation}
\operatorname{Cov} \left( \hat{\boldsymbol{x}}_{0,\mathrm{ideal}}^{\star} \right) = \left( \sum_{i=1}^{n} R_i^{-1} \right)^{-1} \boldsymbol{I} = \frac{c}{ \sum_{i=1}^{n} \operatorname{SNR}(t_i) } \boldsymbol{I},
\label{eq:ideal_consensus_covariance}
\end{equation}
which is strictly smaller, in the positive-semidefinite ordering, than the covariance of any individual observation for $n\geq2$. The proof is provided in Appendix~\ref{sec:proof_variance}.
\end{proposition}

This canonical anchor estimator can be reformulated as the recursive Kalman realization described in the next section. 

\subsection{Connection to Kalman Filtering}\label{sec:kalman}  

The anchor estimator in Eq.~\eqref{eq:ideal_snr_fusion} admits an equivalent recursive realization through static-state Kalman updates. 
Under the independent-view, each observation provides precision-weighted evidence about the shared anchor, allowing the full-history estimator to be implemented recursively.

\begin{theorem}[Recursive Kalman Equivalence under the Ideal Independent-View Model]
\label{thm:kalman_equivalence}
Under the ideal independent-view observation model in Eq.~\eqref{eq:kalman_ideal_observation_model}, initialize the static-state recursive estimator by
\[
\hat{\boldsymbol{x}}_{0,1}^{\operatorname{rec}} = \boldsymbol{y}_{1}, \qquad p_1 = \frac{c}{\operatorname{SNR}(t_1)}.
\]
For $i=2,\ldots,n$, assimilate $\boldsymbol{y}_i$ using the static-state Kalman update with observation covariance $c\,\operatorname{SNR}(t_i)^{-1}\boldsymbol{I}$. Then
\begin{equation}
\hat{\boldsymbol{x}}_{0,n}^{\operatorname{rec}} = \frac{ \sum_{i=1}^{n} \operatorname{SNR}(t_i)\boldsymbol{y}_i }{ \sum_{i=1}^{n} \operatorname{SNR}(t_i) } = \hat{\boldsymbol{x}}_{0,\operatorname{ideal}}^{\star}.
\end{equation}
The proof is provided in Appendix~\ref{sec:proof_kalman}.
\end{theorem} 
Theorem~\ref{thm:kalman_equivalence} shows that the ideal
forward-aligned anchor estimator admits an exact recursive realization
through static-state Kalman updates under the independent-view model.
Motivated by this result, DiFA constructs a causal historical consensus in
practical reverse inference using structural and noise-level compatibility
to account for deviations from the ideal model. The current prediction is
then refined relative to this consensus anchor. 
 
\begin{figure}[t]
    \centering
    \begin{subfigure}[b]{0.49\linewidth}
        \centering
        \centerline{ Baseline Solver  }
        \vspace{4pt} 
        \includegraphics[width=\linewidth]{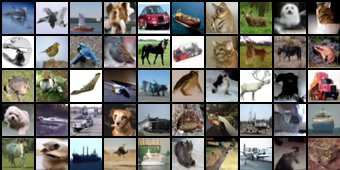}
    \end{subfigure}
    \hfill 
    \begin{subfigure}[b]{0.49\linewidth}
        \centering
        \centerline{ DiFA Refined Prediction}
        \vspace{4pt}
        \includegraphics[width=\linewidth]{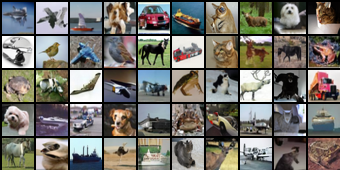}
    \end{subfigure}
    \caption{Visual quality comparison on CIFAR-10 using pre-trained EDM weights. Compared to the baseline solver, incorporating our DiFA framework significantly enhances visual fidelity and edge sharpness under identical few-step inference budgets.}
    \label{fig:dpmdifacifar10}
\end{figure}
 
\subsection{Forward-Aligned Causal Temporal Consensus} 
\label{sec:tcg}

Guided by the forward-aligned anchor interpretation established above, DiFA constructs a practical reference anchor from recently processed clean-signal predictions. We parameterize the noise level using logSNR, $
\ell_i = \log \operatorname{SNR}(t_i) = \log \frac{\alpha_{t_i}^{2}}{\sigma_{t_i}^{2}} 
$. As a monotonic reparameterization of SNR, logSNR preserves the forward-induced reliability ordering while providing improved numerical stability when comparing predictions across a wide dynamic range.   

At step $t_i$, DiFA maintains the causal history buffer:
\begin{equation}
\operatorname{H}_{K}(t_i) = \left\{ \left( \hat{\boldsymbol{x}}_{0}^{(t_j)}, \ell_j \right) \right\}_{t_j\in\operatorname{W}_{K}(t_i)} .
\label{eq:history_buffer}
\end{equation}
Only previously processed predictions are aggregated as historical reference values, so that the current prediction remains available for subsequent anchor-relative refinement. The temporal consensus is defined as:
\begin{equation}
\hat{\boldsymbol{x}}_{0}^{\operatorname{cons}}(t_i) = \sum_{t_j\in\operatorname{W}_{K}(t_i)} a_{ij} \tilde{\boldsymbol{x}}_{0}^{(t_j)}, \sum_{t_j\in\operatorname{W}_{K}(t_i)} a_{ij}=1,
\label{eq:consensus_prediction}
\end{equation}
where $a_{ij}\geq 0$, $\tilde{\boldsymbol{x}}_{0}^{(t_j)}$ denotes an optionally aligned historical prediction. The normalized weights $a_{ij}$ favor historical candidates that are structurally compatible with the current prediction and compatible with its noise level in the logSNR coordinate. The current prediction is used as a query for reference construction, but is not itself aggregated into the reference anchor. This exclusion prevents trivial self-reinforcement and preserves a non-degenerate anchor-relative deviation for subsequent refinement. 

The resulting consensus serves as a causal historical reference anchor in the clean-prediction space. The specific alignment and compatibility-weighting rules used in our experiments are provided in Appendix~\ref{sec:implementation_details}.

\begin{algorithm}[t]
\caption{DiFA: Forward-Process Aligned Inference for Diffusion Models}
\label{alg:difa}
\begin{algorithmic}[1]
\REQUIRE   $D_\theta$, downstream solver,   $\{t_N,\ldots,t_0\}$
\REQUIRE Window size $K$, guidance strength $\omega$
\STATE Initialize history buffer $\operatorname{H}_K \leftarrow \emptyset$
\STATE Sample initial noise $\boldsymbol{x}_{t_N}\sim\mathcal{N}(\boldsymbol{0},\boldsymbol{I})$
\FOR{$i=N,\ldots,1$}
\STATE $\hat{\boldsymbol{x}}_0^{(t_i)} \leftarrow D_\theta(\boldsymbol{x}_{t_i},t_i)$
\STATE $\ell_i \leftarrow \log \operatorname{SNR}(t_i)$
\IF{$\operatorname{H}_K$ is empty}
\STATE $\hat{\boldsymbol{x}}_0^{\mathrm{DiFA}}(t_i) \leftarrow \hat{\boldsymbol{x}}_0^{(t_i)}$
\ELSE
\STATE Construct reference anchor
$\hat{\boldsymbol{x}}_0^{\operatorname{cons}}(t_i)$
using Eq.~\eqref{eq:consensus_prediction}  
\STATE $\boldsymbol{r}_{t_i} \leftarrow \hat{\boldsymbol{x}}_0^{(t_i)} - \hat{\boldsymbol{x}}_0^{\mathrm{cons}}(t_i)$
\STATE Construct deviation guidance
$\boldsymbol{g}_{t_i}$
using Eq.~\eqref{eq:deviation_guidance} 
\STATE $\hat{\boldsymbol{x}}_0^{\mathrm{DiFA}}(t_i) \leftarrow \hat{\boldsymbol{x}}_0^{(t_i)} + \omega\boldsymbol{g}_{t_i}$
\ENDIF
\STATE $\boldsymbol{x}_{t_{i-1}} \leftarrow \operatorname{SolverStep} \left(\boldsymbol{x}_{t_i}, \hat{\boldsymbol{x}}_0^{\mathrm{DiFA}}(t_i), t_i, t_{i-1} \right)$
\STATE Update $\operatorname{H}_K$ with $(\hat{\boldsymbol{x}}_0^{(t_i)},\ell_i)$
\ENDFOR

\RETURN $\boldsymbol{x}_{t_0}$
\end{algorithmic}
\end{algorithm}
\subsection{Anchor-Relative Deviation Guidance}
\label{sec:boost}  
The temporal consensus acts as a multi-step filtered reference that captures information stable across recent reverse steps. Directly substituting this filtered reference for the instantaneous prediction, however, may attenuate components expressed more strongly at the current step. DiFA therefore uses the consensus as a reference anchor and defines the anchor-relative deviation:
\begin{equation}
\boldsymbol{r}_{t_i} = \hat{\boldsymbol{x}}_{0}^{(t_i)} - \hat{\boldsymbol{x}}_{0}^{\operatorname{cons}}(t_i).
\label{eq:anchor_relative_deviation}
\end{equation}
Because this deviation may contain both informative components and
unstable fluctuations, DiFA applies a noise-level-conditioned gain
modulation operator:
\begin{equation}
\boldsymbol{g}_{t_i} = \operatorname{G} \left( \boldsymbol{r}_{t_i}, \hat{\boldsymbol{x}}_{0}^{(t_i)}, \ell_i \right),
\label{eq:deviation_guidance}
\end{equation}
where $\operatorname{G}(\cdot)$ is designed to suppress unreliable deviation components while retaining current-step refinement.

The refined clean-signal prediction is given by:
\begin{equation}
\hat{\boldsymbol{x}}_{0}^{\operatorname{DiFA}}(t_i) = \hat{\boldsymbol{x}}_{0}^{(t_i)} + \omega \boldsymbol{g}_{t_i}, ~\omega\geq 0.
\label{eq:difa_refinement}
\end{equation}
Thus, DiFA performs anchor-relative additive refinement: the temporal consensus defines the reference coordinate, while the modulated deviation provides the correction direction. In the lightweight instantiation evaluated in this work, $\operatorname{G}(\cdot)$ preserves informative residual-oriented components by suppressing prediction-parallel magnitude variation and modulating spatial-frequency deviation components according to the current logSNR. Its complete formulation and all hyperparameter settings are detailed in Appendix~\ref{sec:implementation_details}.

\subsection{Algorithm and Discussion}
\label{sec:algorithm}
Algorithm~\ref{alg:difa} summarizes the complete DiFA procedure. At each timestep, the pretrained denoiser first produces an instantaneous clean-signal prediction $\hat{\boldsymbol{x}}_0^{(t_i)}$. DiFA then refines this
prediction through forward-aligned causal temporal consensus
and anchor-relative deviation guidance. The corrected prediction $\hat{\boldsymbol{x}}_0^{\mathrm{DiFA}}(t_i)$ is finally passed to the downstream solver. Therefore, DiFA does not modify the numerical integration rule itself. Instead, it improves the clean prediction consumed by the solver.

\begin{figure*}[ht] 
    \centering
    \setlength{\tabcolsep}{1pt}    
    \renewcommand{\arraystretch}{0.1} 
    \newcommand{\imgw}{0.088\linewidth}
      
    \newcommand{\mainlabel}[1]{%
        \raisebox{-0.5\height}{\rotatebox{90}{\sffamily\small\textbf{#1}}}%
    }
    \newcommand{\sublabel}[1]{%
        \raisebox{-0.5\height}{\rotatebox{90}{\sffamily\tiny #1}}%
    }
    \newcommand{\vcimg}[1]{%
        \raisebox{-0.5\height}{\includegraphics[width=\imgw]{#1}}%
    }
\begin{tabular}{c c}
         
        \mainlabel{20 NFE} & 
        \begin{tabular}{c c c c c c c c c c c}
            \sublabel{Euler} & \vcimg{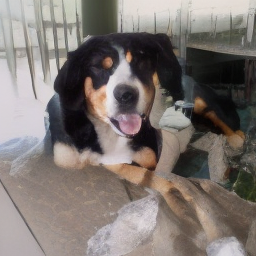} & \vcimg{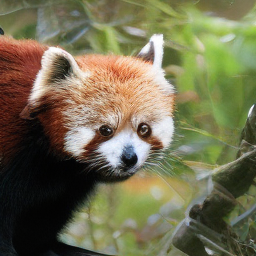} & \vcimg{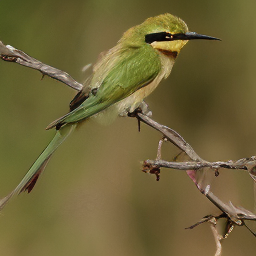} & \vcimg{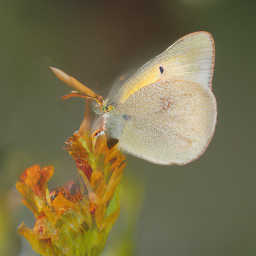} & \vcimg{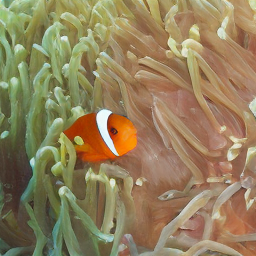} & \vcimg{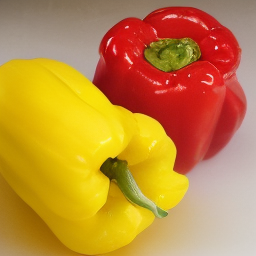} & \vcimg{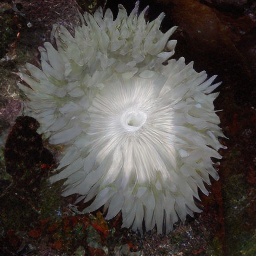} & \vcimg{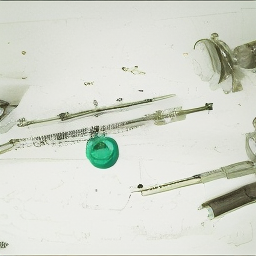} & \vcimg{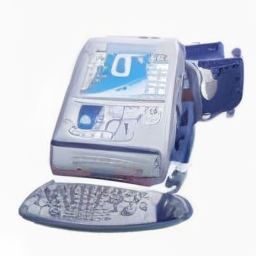} & \vcimg{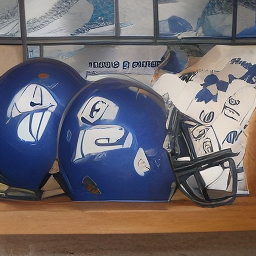} \\
            \sublabel{DiFA}  & \vcimg{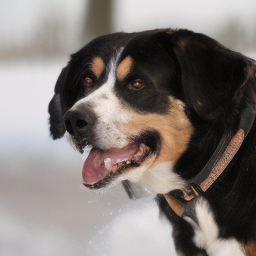} & \vcimg{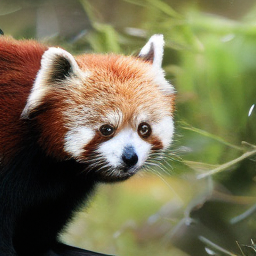} & \vcimg{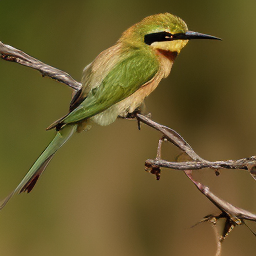} & \vcimg{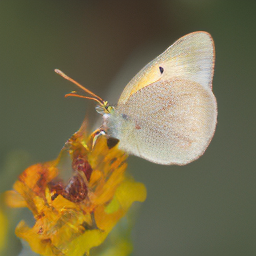} & \vcimg{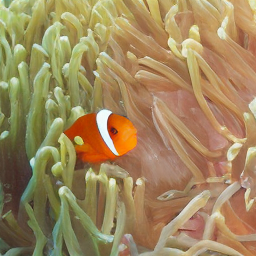} & \vcimg{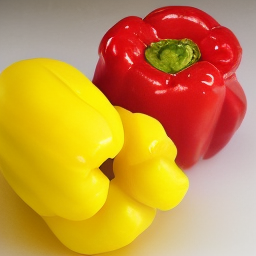} & \vcimg{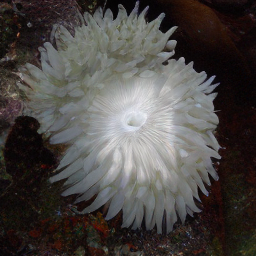} & \vcimg{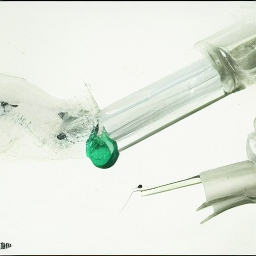} & \vcimg{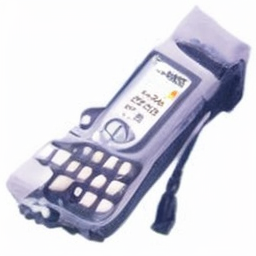} & \vcimg{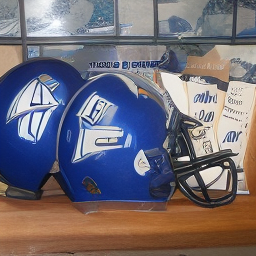} \\\noalign{\vspace{0.8em}}
            
            \sublabel{Heun}  & \vcimg{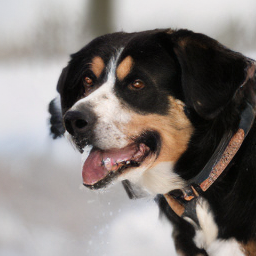} & \vcimg{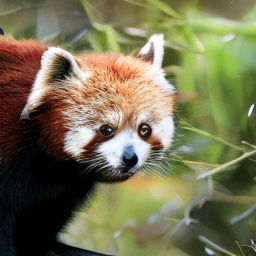} & \vcimg{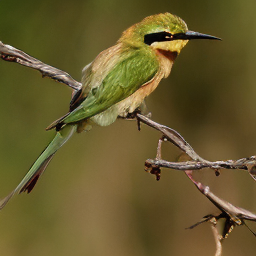} & \vcimg{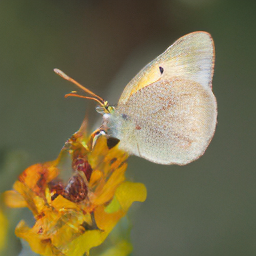} & \vcimg{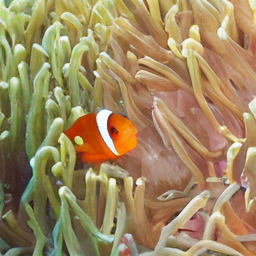} & \vcimg{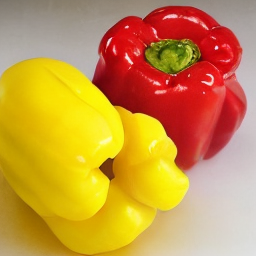} & \vcimg{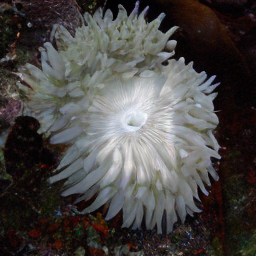} & \vcimg{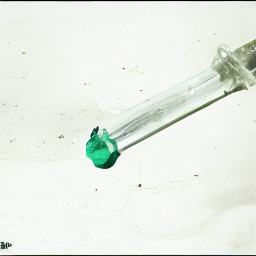} & \vcimg{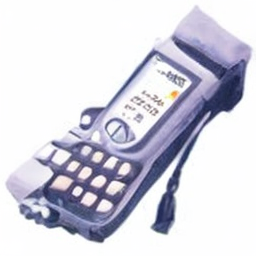} & \vcimg{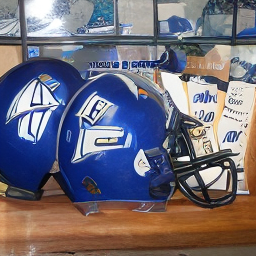} \\
            \sublabel{DiFA}  & \vcimg{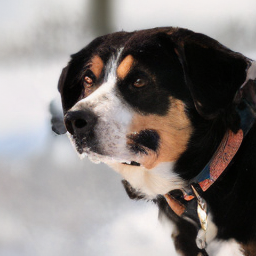} & \vcimg{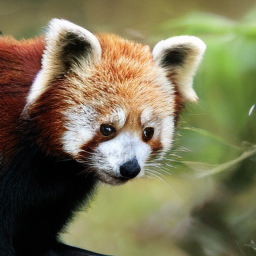} & \vcimg{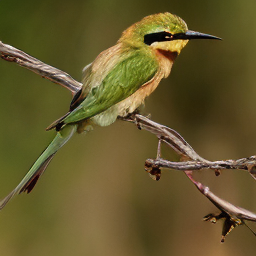} & \vcimg{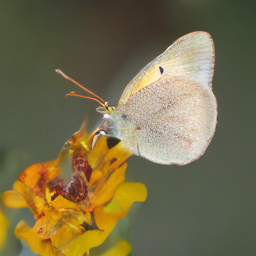} & \vcimg{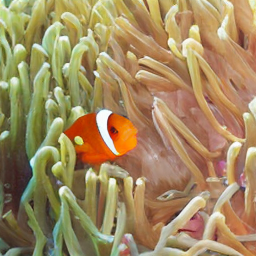} & \vcimg{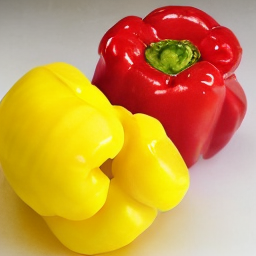} & \vcimg{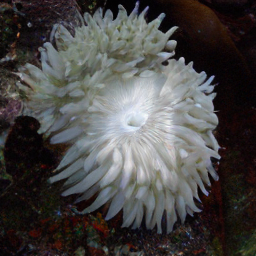} & \vcimg{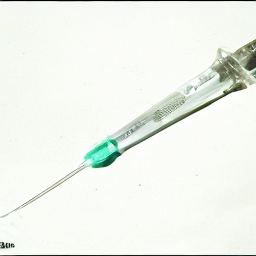} & \vcimg{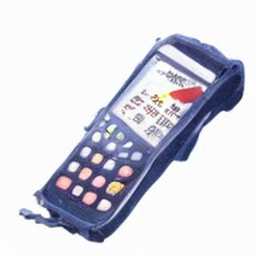} & \vcimg{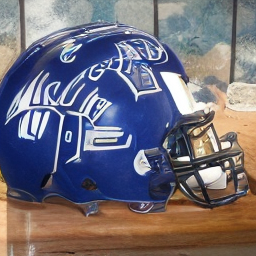} \\
        \end{tabular} \\
        
        \noalign{\vspace{0.4cm}} \hdashline \noalign{\vspace{0.4cm}}
 
        \mainlabel{5 NFE} & 
        \begin{tabular}{c c c c c c c c c c c}
            \sublabel{Euler} & \vcimg{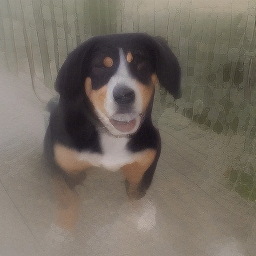} & \vcimg{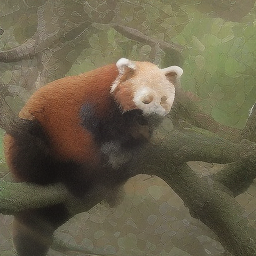} & \vcimg{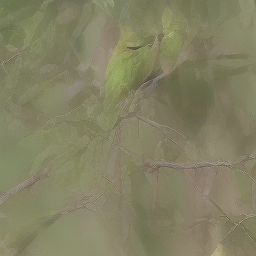} & \vcimg{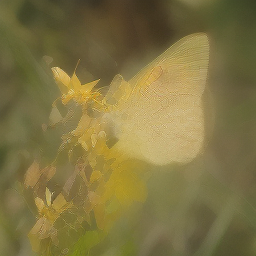} & \vcimg{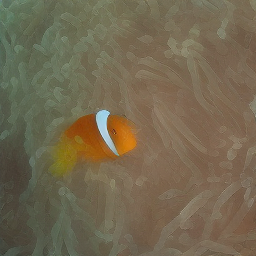} & \vcimg{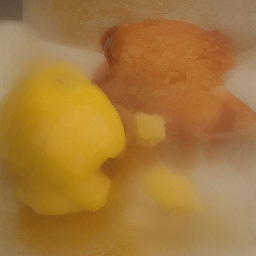} & \vcimg{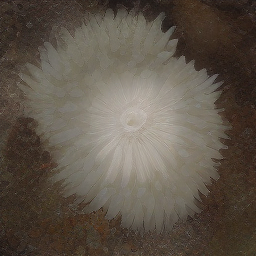} & \vcimg{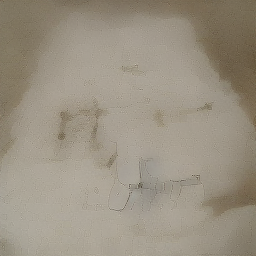} & \vcimg{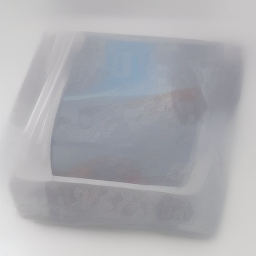} & \vcimg{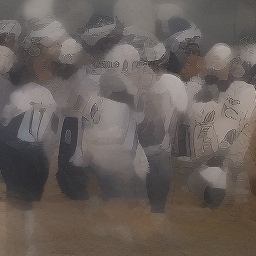} \\ 
            \sublabel{DiFA}  & \vcimg{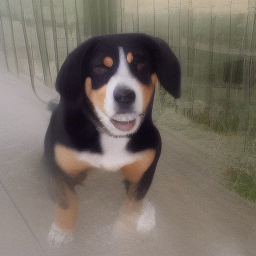} & \vcimg{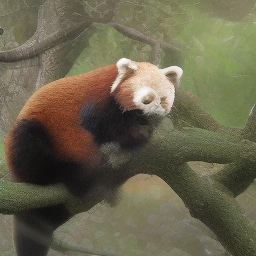} & \vcimg{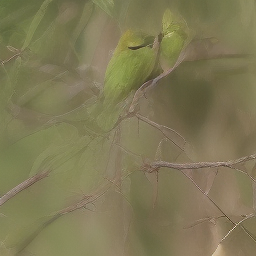} & \vcimg{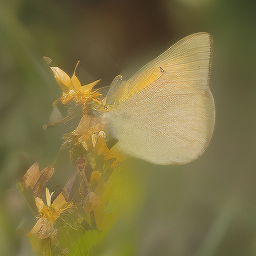} & \vcimg{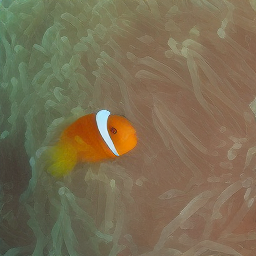} & \vcimg{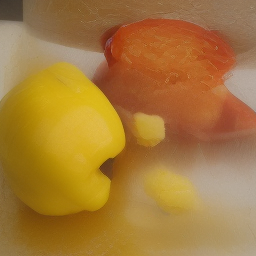} & \vcimg{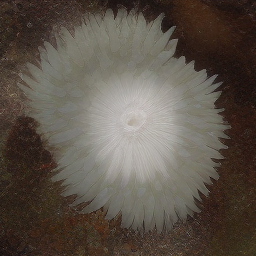} & \vcimg{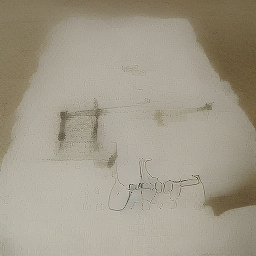} & \vcimg{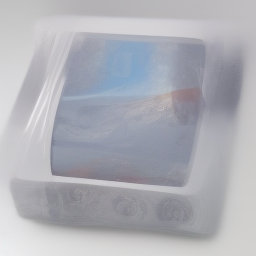} & \vcimg{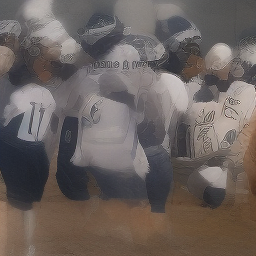} \\\noalign{\vspace{0.8em}}
            
            \sublabel{Heun}  & \vcimg{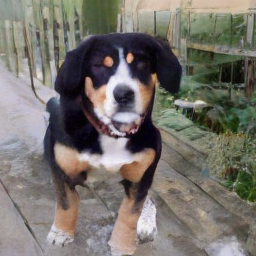} & \vcimg{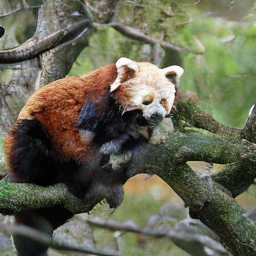} & \vcimg{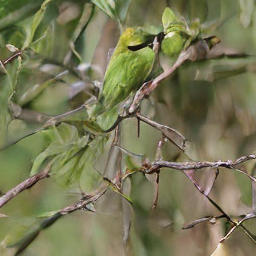} & \vcimg{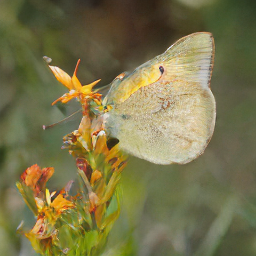} & \vcimg{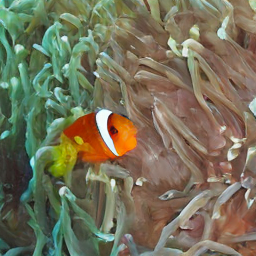} & \vcimg{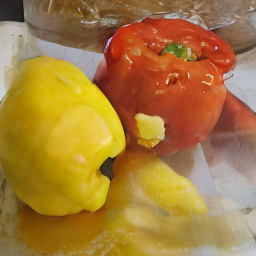} & \vcimg{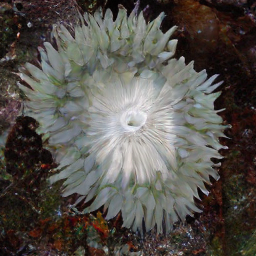} & \vcimg{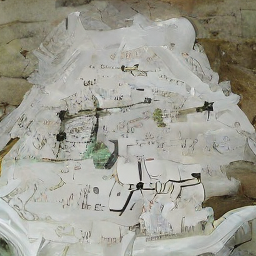} & \vcimg{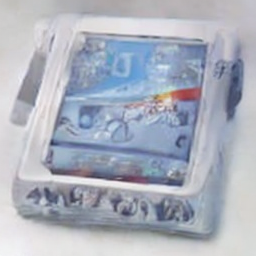} & \vcimg{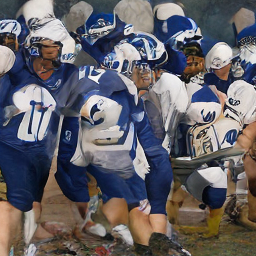} \\ 
            \sublabel{DiFA}  & \vcimg{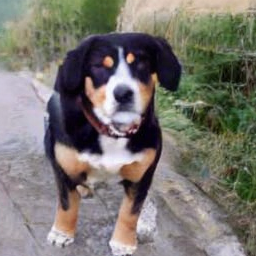} & \vcimg{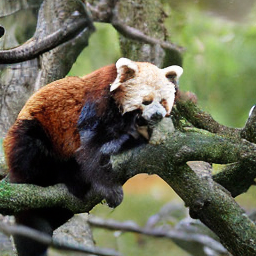} & \vcimg{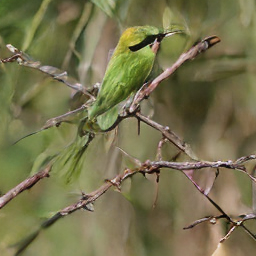} & \vcimg{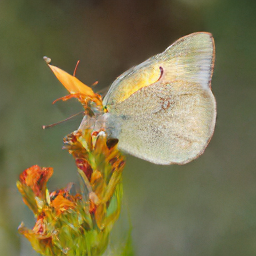} & \vcimg{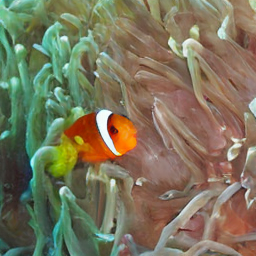} & \vcimg{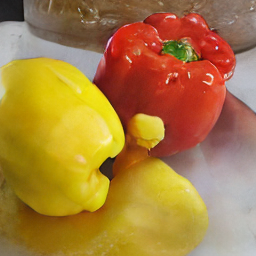} & \vcimg{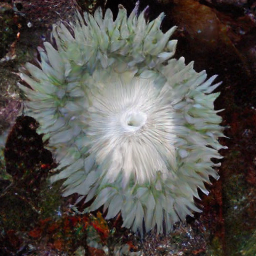} & \vcimg{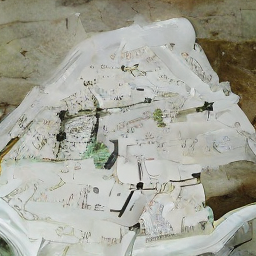} & \vcimg{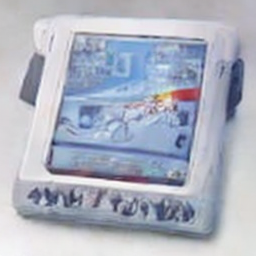} & \vcimg{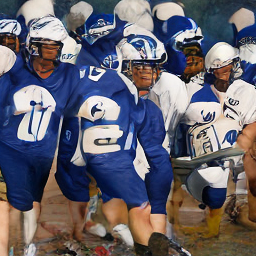} \\
        \end{tabular} \\
    \end{tabular}
    \caption{Visual quality comparison of generated samples on ImageNet 256 × 256 using the pre-trained SiT-XL/2 with a classifier-free guidance (CFG) scale of 1.5. SiT is a DiT-based architecture trained via Flow Matching.  We compare the baseline Euler and Heun samplers against their corresponding refined predictions utilizing our proposed DiFA framework. The results demonstrate that DiFA significantly enhances the visual fidelity and generation quality of the baseline samplers.}
    \label{fig:imagenet256sit}
\end{figure*}

DiFA is solver-compatible because it only changes the clean-signal estimate supplied to the solver. Any solver whose update can be expressed using a clean prediction, or converted to an equivalent clean-prediction parameterization, can use DiFA without changing its integration formula. The method introduces no additional network evaluations. Its extra cost comes only from maintaining a small prediction buffer, computing local similarity weights, and applying lightweight residual operations. With a fixed window size $K$, the additional cost is $\operatorname{O}(Kd)$ per step for a prediction of dimension $d$, which is negligible compared with NFE.

\begin{figure*}[t]
    \centering
    \begin{subfigure}[b]{0.48\textwidth}
        \centering
        \includegraphics[width=\linewidth]{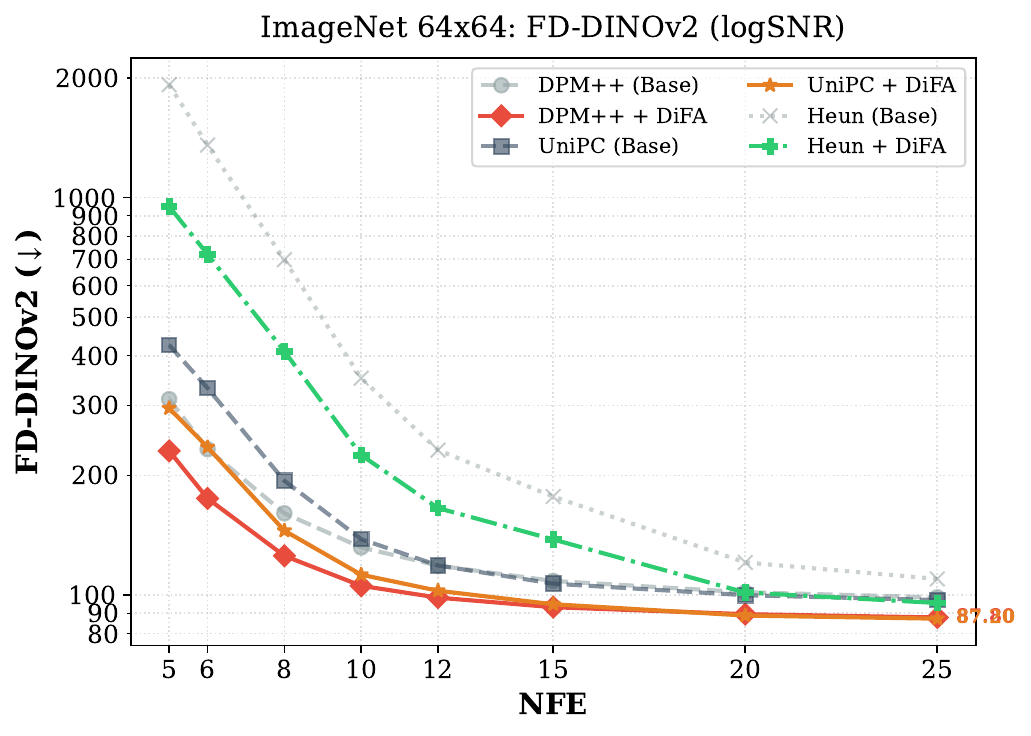}
        \caption{ImageNet: FD-DINOv2 (logSNR)}
    \end{subfigure}
    \hfill
    \begin{subfigure}[b]{0.48\textwidth}
        \centering
        \includegraphics[width=\linewidth]{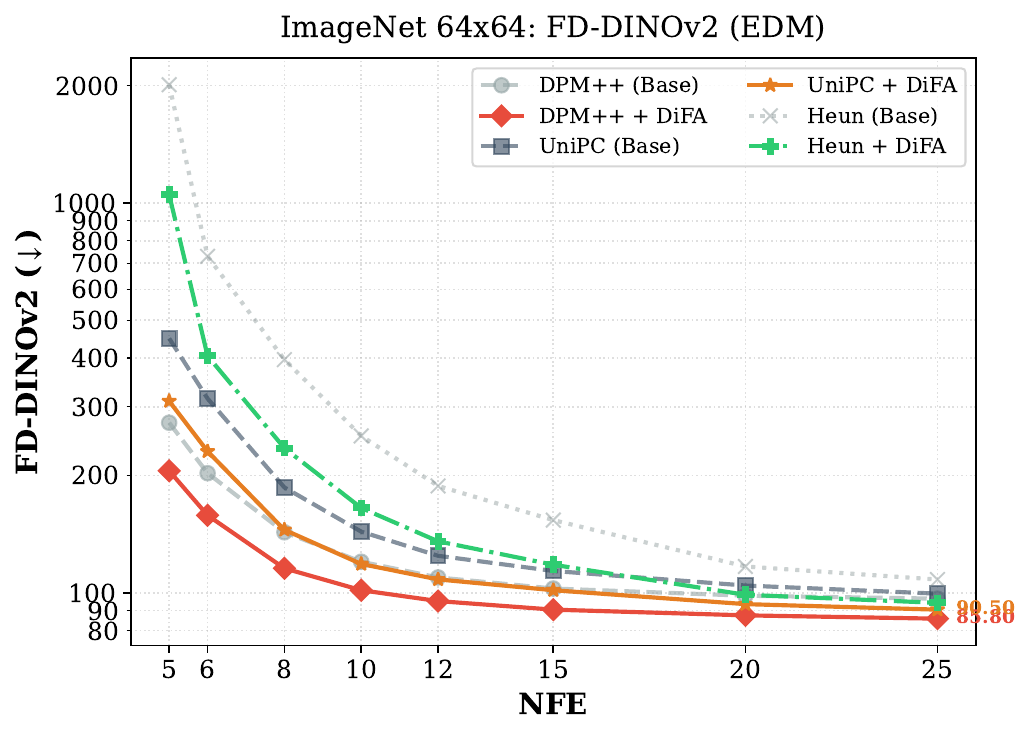}
        \caption{ImageNet: FD-DINOv2 (EDM)}
    \end{subfigure}
    \caption{Perceptual Quality (FD-DINOv2). DiFA also yields consistent improvements in FD-DINOv2 scores, indicating that our method enhances not just statistical fidelity (FID) but also perceptually aligned image quality.}
    \label{fig:dino_results}
\end{figure*}

\begin{table}[t]
\centering
\caption{Quantitative results on CIFAR-10 ($32\times32$). We compare DiFA against state-of-the-art GANs and diffusion models. DiFA (Ours) achieves competitive performance comparable to distillation methods (e.g., CTM) while requiring zero training. $\dagger$: Methods requiring training/distillation.}
\label{tab:cifar10_sota}
\resizebox{0.45\textwidth}{!}{ 
\begin{tabular}{llccc}
\toprule
Type & Model & \#Params & NFE & FID $\downarrow$ \\ \midrule
\multirow{5}{*}{GAN$^\dagger$} 
 & StyleGAN2-ADA \cite{karras2020training} & 20M & 1 & 2.92 \\
 & StyleGAN-XL \cite{sauer2022stylegan} & 18M & 1 & 1.85 \\
 & R3GAN \cite{huang2025r3gan} & 21M & 1 & 1.96 \\
 & CTM \cite{kim2023consistency} & 59M & 1 & 1.98 \\
 & SiD$^2$A \cite{zhou2024sid2a} & 56M & 1 & 1.50 \\ \midrule
\multirow{9}{*}{Diffusion} 
 & DDPM \cite{ho2020denoising} & 36M & 1000 & 3.17 \\
 & iDDPM \cite{nichol2021improved} & 50M & 4000 & 2.90 \\
 & DDIM \cite{song2021denoising} & 36M & 100 & 4.16 \\
 & DPM-Solver \cite{lu2022dpm} & 36M & 48 & 2.65 \\
 & DPM-Solver-v3 \cite{zheng2023dpm} & 36M & 25 & 2.00 \\
 & NCSN++ \cite{song2020score} & 108M & 2000 & 2.20 \\
 & LSGM \cite{vahdat2021score} & 376M & 138 & 2.10 \\
 & EDM \cite{karras2022elucidating} & 56M & 35 & 1.97 \\
 & EVODiff \cite{li2025evodiff} & 36M & 15 & 2.06 \\ \midrule
\multirow{2}{*}{Ours} 
 & DPM-Solver++ w/ DiFA & 36M & 15 & 2.02 \\
 & DPM-Solver++ w/ DiFA & 36M & 20 & \emph{1.96} \\ \bottomrule
\end{tabular}%
}
\end{table}

\begin{table}[t]
\centering
\caption{Results on class-conditional ImageNet-64. DiFA significantly outperforms standard diffusion baselines (e.g., ADM, iDDPM) with an order of magnitude fewer steps.}
\label{tab:imagenet_sota}
\resizebox{0.45\textwidth}{!}{ 
\begin{tabular}{llccc}
\toprule
Type & Model & \#Params & NFE & FID $\downarrow$ \\ \midrule
\multirow{5}{*}{GAN$^\dagger$} 
 & StyleGAN-XL \cite{sauer2022stylegan} & 135M & 1 & 1.51 \\
 & R3GAN \cite{huang2025r3gan} & 104M & 1 & 2.09 \\
 & CTM \cite{kim2023consistency} & 324M & 1 & 1.92 \\
 & DMD2 \cite{yin2024one} & 296M & 1 & 1.28 \\
 & SiD$^2$A \cite{zhou2024sid2a} & 296M & 1 & 1.11 \\ \midrule
\multirow{9}{*}{Diffusion} 
 & iDDPM \cite{nichol2021improved} & 270M & 250 & 2.92 \\
 & ADM \cite{dhariwal2021diffusion} & 296M & 250 & 2.07 \\
 & RIN \cite{jabri2022scalable} & 281M & 1000 & 1.23 \\
 & EDM \cite{karras2022elucidating} & 296M & 511 & 1.36 \\
 & VDM++ \cite{kingma2023understanding} & 296M & 511 & 1.43 \\
 & DisCo-Diff \cite{xu2024disco} & - & 623 & 1.22 \\
 & EDM2-S \cite{karras2024analyzing} & 280M & 63 & 1.58 \\ 
 & DPM-Solver++ (Base) & 296M & 25 & 1.83 \\ 
 & UniPC (Base) & 296M & 25 & 1.73 \\ \midrule
\multirow{2}{*}{Ours} 
 & DPM-Solver++ w/ DiFA & 296M & 25 & 1.64 \\
 & UniPC w/ DiFA & 296M & 25 & \emph{1.63} \\ \bottomrule
\end{tabular}%
}
\end{table}

\section{Experiments}
\label{sec:exp}

\subsection{Settings}
\noindent \textbf{Datasets \& Metrics.} We evaluate DiFA on widely adopted benchmarks: CIFAR-10 ($32\times32$) and ImageNet-64 ($64\times64$). To ensure a holistic assessment, we employ three complementary metrics: Fréchet Inception Distance (FID) \cite{Heusel2017gans} for distributional fidelity, Inception Score (IS) \cite{salimans2016improved} for sample diversity and clarity, and FD-DINOv2 \cite{stein2023exposing} for perceptually aligned quality assessment. 

\noindent \textbf{Implementation Details.} 
We use official pretrained diffusion checkpoints and evaluate DiFA with DDIM \cite{song2021denoising}, DPM-Solver++ \cite{lu2025dpm}, UniPC \cite{zhao2023unipc}, and Heun under the corresponding VP/logSNR and EDM parameterizations \cite{karras2022elucidating}.  
All experiments are conducted on NVIDIA RTX 4090 GPUs, 
except for the CFG-enabled systematic evaluation reported in
Table~\ref{tab:difa_sitxl2_imagenet256}, which is conducted on
NVIDIA H200 GPUs. NFE denotes the number of neural network
evaluations. 
Crucially, DiFA is implemented as a plug-and-play wrapper around existing solvers (DDIM, DPM-Solver++, UniPC, Heun) without any additional training or model fine-tuning. We use the same initial noise seeds for baselines and DiFA to ensure strict fairness. Code is available at \href{[https://github.com/ShiguiLi/DiFA](https://github.com/ShiguiLi/DiFA)}{DiFA}.

\subsection{Main Results}

We now present a comprehensive evaluation demonstrating that DiFA consistently improves the evaluated training-free solvers across various regimes, effectively bridging the gap between rapid prototyping and high-fidelity generation.

\noindent \textbf{Surpassing the Efficiency Barrier on CIFAR-10.}
As illustrated in the top row of Figure~\ref{fig:main_results}, the
qualitative comparison in Figure~\ref{fig:dpmdifacifar10}, and the SOTA
comparison in Table~\ref{tab:cifar10_sota}, DiFA substantially improves
the trade-off between quality and computational cost.  In the extremely low-budget regime (5--8 steps), where standard high-order solvers typically suffer from severe discretization errors due to unstable derivative estimation, DiFA effectively mitigates these issues via its temporal consensus mechanism. For instance, at NFE 8, DiFA reduces the FID of DPM-Solver++ from 8.40 to \textbf{4.15}, achieving a relative improvement of over \textbf{50\%} and turning severely degraded few-step outputs into high-quality samples. Moving to the mid-NFE regime (10--15 steps), which is critical for real-world applications, DiFA exhibits clear dominance. At \textbf{NFE 12}, it achieves an FID of \textbf{2.18} (compared to the baseline's 3.70), surpassing the baseline by a large margin and outperforming the reported 12-step result of EVODiff (FID 2.25). Furthermore, even as the number of steps increases to 20 where baseline solvers typically saturate, DiFA continues to yield gains (FID 1.96 vs. 2.33). This indicates that orthogonal deviation guidance helps preserve residual details that may otherwise be attenuated by consensus-based stabilization, allowing the model to break the FID 2.0 barrier without extra training.   

\noindent \textbf{Universal Consistency on ImageNet.}
The bottom row of Figure~\ref{fig:main_results} and Table~\ref{tab:imagenet_sota}  highlight the robustness of DiFA on the more complex ImageNet manifold. The proposed framework shows broad solver compatibility: DiFA consistently lowers the FID curve when coupled with either the predictor-corrector-based UniPC or the high-order DPM-Solver++.   Notably, DiFA substantially improves the unstable low-step behavior of the Heun solver, reducing FID from 230.05 to 110.20 at NFE 5. This suggests that the proposed prediction-alignment mechanism can mitigate severe trajectory drift in challenging low-NFE regimes.  Moreover, DiFA proves effective under both logSNR and EDM noise schedules, suggesting that the proposed prediction-alignment principle is not tied to a specific discretization choice. At convergence (NFE 25), DiFA reaches an FID of \textbf{1.63--1.64} on ImageNet-64, showing strong training-free sampling performance among the evaluated solver configurations. This result suggests that refining clean-signal predictions can better exploit the capacity of pretrained diffusion models without additional network evaluations.

\noindent \textbf{Perceptual Quality Enhancement.}
Beyond FID, FD-DINOv2 provides a perceptually aligned evaluation of
image quality. Figure~\ref{fig:dino_results} demonstrates that DiFA yields consistent improvements in FD-DINOv2 scores. Unlike standard averaging methods that might improve FID at the cost of blurriness, DiFA achieves perceptual fidelity through a dual mechanism for structural filtering and texture enhancement, suggesting that the improvement is not obtained merely by sacrificing
perceptual sharpness or semantic coherence.

To further examine whether DiFA is tied to pixel-space diffusion or a specific solver family, we provide further component, robustness, and cross-paradigm studies in
Appendix~\ref{app:ablation}. On LSUN Bedroom, as shown in Table~\ref{tab:ldm_main}, DiFA  
improves a latent-diffusion baseline across 5, 10, and 20 NFE,
demonstrating its effectiveness beyond the main pixel-space setting.
Controlled EDM ablations show that historical consensus provides the
main contribution, while SNR gating and magnitude alignment offer
additional stabilization. Sensitivity analyses of the temporal-window,
SNR-gating, residual-modulation, and logSNR-compatibility parameters
further confirm robustness across different hyperparameter settings.
Finally, experiments on SiT-XL/2 with ImageNet $256\times256$ show
consistent FID and IS improvements under both unguided and CFG-enabled
flow matching settings; qualitative examples are shown in
Figure~\ref{fig:imagenet256sit}.

\section*{Conclusion and Limitations}
We presented DiFA, a training-free inference framework
that reframes clean-signal prediction refinement as a sequential
state-estimation problem. Rather than treating each denoising prediction as an exact estimate, DiFA exploits the common-anchor geometry and noise-level-dependent reliability structure induced by the forward process. It constructs a causal temporal consensus from historical predictions and applies deviation guidance to preserve informative residual details. Experiments across diffusion and flow-matching settings demonstrate that DiFA consistently improves sampling quality without additional network evaluations, supporting the effectiveness and broader applicability of prediction alignment. The current framework nevertheless relies on clean-prediction parameterizations, fixed hyperparameters, and an idealized static-anchor assumption. Future work will explore adaptive strategies, improved filtering mechanisms, and theoretical extensions for temporally correlated errors, as well as applications to text-to-image, video generation, inverse problems and distillation pipelines.

\section*{Acknowledgements} 
This work was supported in part by grants from National
Natural Science Foundation of China (52539005), the China
Scholarship Council (202306150167), the fundamental research program of Guangdong, China (2023A1515011281),
Guangdong Basic and Applied Basic Research Foundation
(24202107190000687), Foshan Science and Technology
Research Project (2220001018608).

\section*{Impact Statement} 
This work introduces a training-free inference-time paradigm for improving generative-model sampling by refining clean-signal predictions rather than retraining models or replacing numerical solvers. By exploiting temporal prediction redundancy, DiFA can reduce the computational cost of high-quality generation and may make strong diffusion and flow-matching models more accessible under limited computational budgets. At the same time, more efficient generation may also lower the barrier to misuse, including the creation of misleading or harmful synthetic content. Since DiFA does not introduce new training data or modify model weights, these risks are largely inherited from the underlying generative models. We therefore encourage DiFA to be used together with existing safety filters, provenance mechanisms, and responsible deployment practices. 


\nocite{langley00}

\bibliography{example_paper}

\begin{thebibliography}{95}
\providecommand{\natexlab}[1]{#1}
\providecommand{\url}[1]{\texttt{#1}}
\expandafter\ifx\csname urlstyle\endcsname\relax
  \providecommand{\doi}[1]{doi: #1}\else
  \providecommand{\doi}{doi: \begingroup \urlstyle{rm}\Url}\fi

\bibitem[Abuduweili et~al.(2025)Abuduweili, Yuan, Liu, and
  Permenter]{abuduweili2025enhancing}
Abuduweili, A., Yuan, C., Liu, C., and Permenter, F.
\newblock Enhancing sample generation of diffusion models using noise level
  correction.
\newblock \emph{Transactions on Machine Learning Research}, 2025.
\newblock ISSN 2835-8856.
\newblock URL \url{https://openreview.net/forum?id=y8VXikiIU0}.

\bibitem[Ahn et~al.(2024)Ahn, Cho, Min, Jang, Kim, Kim, Park, Jin, and
  Kim]{ahn2024self}
Ahn, D., Cho, H., Min, J., Jang, W., Kim, J., Kim, S., Park, H.~H., Jin, K.~H.,
  and Kim, S.
\newblock Self-rectifying diffusion sampling with perturbed-attention guidance.
\newblock In \emph{European Conference on Computer Vision}, pp.\  1--17.
  Springer, 2024.

\bibitem[Bai et~al.(2025)Bai, Shao, Zhou, Qi, Xu, Xiong, and
  Xie]{bai2025zigzag}
Bai, L., Shao, S., Zhou, Z., Qi, Z., Xu, Z., Xiong, H., and Xie, Z.
\newblock Zigzag diffusion sampling: Diffusion models can self-improve via
  self-reflection.
\newblock In \emph{The Thirteenth International Conference on Learning
  Representations}, 2025.
\newblock URL \url{https://openreview.net/forum?id=MKvQH1ekeY}.

\bibitem[Bansal et~al.(2024)Bansal, Chu, Schwarzschild, Sengupta, Goldblum,
  Geiping, and Goldstein]{bansal2024universal}
Bansal, A., Chu, H.-M., Schwarzschild, A., Sengupta, R., Goldblum, M., Geiping,
  J., and Goldstein, T.
\newblock Universal guidance for diffusion models.
\newblock In \emph{The Twelfth International Conference on Learning
  Representations}, 2024.
\newblock URL \url{https://openreview.net/forum?id=pzpWBbnwiJ}.

\bibitem[Bao et~al.(2022)Bao, Li, Zhu, and Zhang]{bao2022analyticdpm}
Bao, F., Li, C., Zhu, J., and Zhang, B.
\newblock Analytic-{DPM}: an analytic estimate of the optimal reverse variance
  in diffusion probabilistic models.
\newblock In \emph{International Conference on Learning Representations}, 2022.
\newblock URL \url{https://openreview.net/forum?id=0xiJLKH-ufZ}.

\bibitem[Batifol et~al.(2025)Batifol, Blattmann, Boesel, Consul, Diagne,
  Dockhorn, English, English, Esser, Kulal, et~al.]{batifol2025flux}
Batifol, S., Blattmann, A., Boesel, F., Consul, S., Diagne, C., Dockhorn, T.,
  English, J., English, Z., Esser, P., Kulal, S., et~al.
\newblock Flux. 1 kontext: Flow matching for in-context image generation and
  editing in latent space.
\newblock \emph{arXiv e-prints}, pp.\  arXiv--2506, 2025.

\bibitem[Blattmann et~al.(2023)Blattmann, Rombach, Ling, Dockhorn, Kim, Fidler,
  and Kreis]{blattmann2023align}
Blattmann, A., Rombach, R., Ling, H., Dockhorn, T., Kim, S.~W., Fidler, S., and
  Kreis, K.
\newblock Align your latents: High-resolution video synthesis with latent
  diffusion models.
\newblock In \emph{Proceedings of the IEEE/CVF conference on computer vision
  and pattern recognition}, pp.\  22563--22575, 2023.

\bibitem[Chang et~al.(2026)Chang, Koulieris, Chang, and Shum]{chang2026design}
Chang, Z., Koulieris, G.~A., Chang, H.~J., and Shum, H.~P.
\newblock On the design fundamentals of diffusion models: A survey.
\newblock \emph{Pattern Recognition}, 169:\penalty0 111934, 2026.

\bibitem[Chen et~al.(2025)Chen, Li, Li, Yang, Paisley, and
  Zeng]{chen2025dequantified}
Chen, W., Li, S., Li, J., Yang, J., Paisley, J., and Zeng, D.
\newblock Dequantified diffusion-schr\"odinger bridge for density ratio
  estimation.
\newblock In \emph{Forty-second International Conference on Machine Learning},
  2025.
\newblock URL \url{https://openreview.net/forum?id=zvyHCOcwsw}.

\bibitem[Chen et~al.(2026)Chen, Du, Li, Zeng, and Paisley]{CHEN2026112442}
Chen, W., Du, S., Li, S., Zeng, D., and Paisley, J.
\newblock Entropy-informed weighting channel normalizing flow for deep
  generative models.
\newblock \emph{Pattern Recognition}, 172:\penalty0 112442, 2026.
\newblock ISSN 0031-3203.
\newblock \doi{https://doi.org/10.1016/j.patcog.2025.112442}.

\bibitem[Chen et~al.(2024)Chen, Li, Wang, Zhang, Xu, Jiang, Song, and
  Wang]{chen2024rethinking}
Chen, Z., Li, H., Wang, F., Zhang, O., Xu, H., Jiang, X., Song, Z., and Wang,
  H.
\newblock Rethinking the diffusion models for missing data imputation: A
  gradient flow perspective.
\newblock \emph{Advances in Neural Information Processing Systems},
  37:\penalty0 112050--112103, 2024.

\bibitem[Chung et~al.(2022)Chung, Sim, Ryu, and Ye]{chung2022improving}
Chung, H., Sim, B., Ryu, D., and Ye, J.~C.
\newblock Improving diffusion models for inverse problems using manifold
  constraints.
\newblock \emph{Advances in Neural Information Processing Systems},
  35:\penalty0 25683--25696, 2022.

\bibitem[Dhariwal \& Nichol(2021)Dhariwal and Nichol]{dhariwal2021diffusion}
Dhariwal, P. and Nichol, A.
\newblock Diffusion models beat gans on image synthesis.
\newblock In \emph{Advances in Neural Information Processing Systems},
  volume~34, pp.\  8780--8794, 2021.

\bibitem[Dong et~al.(2024)Dong, Liang, Li, Zhang, Cao, Ding, Khan, and
  Shahbaz~Khan]{dong2024continually}
Dong, J., Liang, W., Li, H., Zhang, D., Cao, M., Ding, H., Khan, S.~H., and
  Shahbaz~Khan, F.
\newblock How to continually adapt text-to-image diffusion models for flexible
  customization?
\newblock \emph{Advances in Neural Information Processing Systems},
  37:\penalty0 130057--130083, 2024.

\bibitem[Epstein et~al.(2023)Epstein, Jabri, Poole, Efros, and
  Holynski]{epstein2023diffusion}
Epstein, D., Jabri, A., Poole, B., Efros, A., and Holynski, A.
\newblock Diffusion self-guidance for controllable image generation.
\newblock \emph{Advances in Neural Information Processing Systems},
  36:\penalty0 16222--16239, 2023.

\bibitem[Esser et~al.(2024)Esser, Kulal, Blattmann, Entezari, M{\"u}ller,
  Saini, Levi, Lorenz, Sauer, Boesel, Podell, Dockhorn, English, and
  Rombach]{esser2024scaling}
Esser, P., Kulal, S., Blattmann, A., Entezari, R., M{\"u}ller, J., Saini, H.,
  Levi, Y., Lorenz, D., Sauer, A., Boesel, F., Podell, D., Dockhorn, T.,
  English, Z., and Rombach, R.
\newblock Scaling rectified flow transformers for high-resolution image
  synthesis.
\newblock In \emph{Forty-first International Conference on Machine Learning},
  2024.
\newblock URL \url{https://openreview.net/forum?id=FPnUhsQJ5B}.

\bibitem[Farghly et~al.(2025)Farghly, Potaptchik, Howard, Deligiannidis, and
  Pidstrigach]{farghly2025diffusion}
Farghly, T., Potaptchik, P., Howard, S., Deligiannidis, G., and Pidstrigach, J.
\newblock Diffusion models and the manifold hypothesis: Log-domain smoothing is
  geometry adaptive.
\newblock In \emph{The Thirty-ninth Annual Conference on Neural Information
  Processing Systems}, 2025.
\newblock URL \url{https://openreview.net/forum?id=4JihzQXNJn}.

\bibitem[Frans et~al.(2025)Frans, Hafner, Levine, and Abbeel]{frans2025one}
Frans, K., Hafner, D., Levine, S., and Abbeel, P.
\newblock One step diffusion via shortcut models.
\newblock In \emph{The Thirteenth International Conference on Learning
  Representations}, 2025.
\newblock URL \url{https://openreview.net/forum?id=OlzB6LnXcS}.

\bibitem[Fu et~al.(2025)Fu, Wang, Chen, Shen, and Tao]{fu2025selfverification}
Fu, S., Wang, Y., Chen, Y., Shen, L., and Tao, D.
\newblock Self-verification provably prevents model collapse in recursive
  synthetic training.
\newblock In Belgrave, D., Zhang, C., Lin, H., Pascanu, R., Koniusz, P.,
  Ghassemi, M., and Chen, N. (eds.), \emph{Advances in Neural Information
  Processing Systems}, volume~38, pp.\  36101--36154. Curran Associates, Inc.,
  2025.
\newblock URL \url{https://openreview.net/forum?id=X5Hk8aMs6w}.

\bibitem[Geng et~al.(2025)Geng, Deng, Bai, Kolter, and He]{geng2025mean}
Geng, Z., Deng, M., Bai, X., Kolter, J.~Z., and He, K.
\newblock Mean flows for one-step generative modeling.
\newblock In \emph{The Thirty-ninth Annual Conference on Neural Information
  Processing Systems}, 2025.
\newblock URL \url{https://openreview.net/forum?id=uWj4s7rMnR}.

\bibitem[Gonzalez et~al.(2023)Gonzalez, Fernandez~Pinto, Tran, Hajri, Masmoudi,
  et~al.]{gonzalez2023seeds}
Gonzalez, M., Fernandez~Pinto, N., Tran, T., Hajri, H., Masmoudi, N., et~al.
\newblock Seeds: Exponential sde solvers for fast high-quality sampling from
  diffusion models.
\newblock \emph{Advances in Neural Information Processing Systems},
  36:\penalty0 68061--68120, 2023.

\bibitem[Goodfellow et~al.(2014)Goodfellow, Pouget-Abadie, Mirza, Xu,
  Warde-Farley, Ozair, Courville, and Bengio]{NIPS2014_5ca3e9b1}
Goodfellow, I.~J., Pouget-Abadie, J., Mirza, M., Xu, B., Warde-Farley, D.,
  Ozair, S., Courville, A., and Bengio, Y.
\newblock Generative adversarial networks, 2014.
\newblock URL \url{https://arxiv.org/abs/1406.2661}.

\bibitem[Heusel et~al.(2017)Heusel, Ramsauer, Unterthiner, Nessler, and
  Hochreiter]{Heusel2017gans}
Heusel, M., Ramsauer, H., Unterthiner, T., Nessler, B., and Hochreiter, S.
\newblock Gans trained by a two time-scale update rule converge to a local nash
  equilibrium.
\newblock \emph{Advances in neural information processing systems}, 30, 2017.

\bibitem[Ho \& Salimans(2022)Ho and Salimans]{ho2021classifier}
Ho, J. and Salimans, T.
\newblock Classifier-free diffusion guidance, 2022.
\newblock URL \url{https://arxiv.org/abs/2207.12598}.

\bibitem[Ho et~al.(2020)Ho, Jain, and Abbeel]{ho2020denoising}
Ho, J., Jain, A., and Abbeel, P.
\newblock Denoising diffusion probabilistic models.
\newblock In \emph{Advances in Neural Information Processing Systems},
  volume~33, pp.\  6840--6851, 2020.

\bibitem[Hong et~al.(2023)Hong, Lee, Jang, and Kim]{hong2023improving}
Hong, S., Lee, G., Jang, W., and Kim, S.
\newblock Improving sample quality of diffusion models using self-attention
  guidance.
\newblock In \emph{Proceedings of the IEEE/CVF International Conference on
  Computer Vision}, pp.\  7462--7471, 2023.

\bibitem[Huang et~al.(2025)Huang, Gokaslan, Kuleshov, and
  Tompkin]{huang2025r3gan}
Huang, Y., Gokaslan, A., Kuleshov, V., and Tompkin, J.
\newblock The gan is dead; long live the gan! a modern gan baseline.
\newblock \emph{arXiv preprint arXiv:2501.05441}, 2025.

\bibitem[Jabri et~al.(2022)Jabri, Fleet, and Chen]{jabri2022scalable}
Jabri, A., Fleet, D., and Chen, T.
\newblock Scalable adaptive computation for iterative generation.
\newblock \emph{arXiv preprint arXiv:2212.11972}, 2022.

\bibitem[Jia et~al.(2025)Jia, Liu, Song, Yuan, Shen, and
  Wang]{jia2025antithetic}
Jia, J., Liu, S., Song, B., Yuan, W., Shen, L., and Wang, G.
\newblock Antithetic noise in diffusion models.
\newblock \emph{arXiv preprint arXiv:2506.06185}, 2025.

\bibitem[Jiang et~al.(2025)Jiang, Wang, Ye, Shao, Sun, Zhang, Chen, Zhang,
  Chen, and Li]{jiang2025sada}
Jiang, T., Wang, Y., Ye, H., Shao, Z., Sun, J., Zhang, J., Chen, Z., Zhang, J.,
  Chen, Y., and Li, H.
\newblock {SADA}: Stability-guided adaptive diffusion acceleration.
\newblock In \emph{Forty-second International Conference on Machine Learning},
  2025.
\newblock URL \url{https://openreview.net/forum?id=ThMQfsBnje}.

\bibitem[Karras et~al.(2020)Karras, Aittala, Hellsten, Laine, Lehtinen, and
  Aila]{karras2020training}
Karras, T., Aittala, M., Hellsten, J., Laine, S., Lehtinen, J., and Aila, T.
\newblock Training generative adversarial networks with limited data.
\newblock In \emph{Advances in Neural Information Processing Systems},
  volume~33, pp.\  12104--12114, 2020.

\bibitem[Karras et~al.(2022)Karras, Aittala, Aila, and
  Laine]{karras2022elucidating}
Karras, T., Aittala, M., Aila, T., and Laine, S.
\newblock Elucidating the design space of diffusion-based generative models.
\newblock \emph{Advances in neural information processing systems},
  35:\penalty0 26565--26577, 2022.

\bibitem[Karras et~al.(2024{\natexlab{a}})Karras, Aittala,
  Kynk{\"a}{\"a}nniemi, Lehtinen, Aila, and Laine]{karras2024guiding}
Karras, T., Aittala, M., Kynk{\"a}{\"a}nniemi, T., Lehtinen, J., Aila, T., and
  Laine, S.
\newblock Guiding a diffusion model with a bad version of itself.
\newblock \emph{Advances in Neural Information Processing Systems},
  37:\penalty0 52996--53021, 2024{\natexlab{a}}.

\bibitem[Karras et~al.(2024{\natexlab{b}})Karras, Aittala, Lehtinen, Hellsten,
  Aila, and Laine]{karras2024analyzing}
Karras, T., Aittala, M., Lehtinen, J., Hellsten, J., Aila, T., and Laine, S.
\newblock Analyzing and improving the training dynamics of diffusion models.
\newblock In \emph{Proceedings of the IEEE/CVF Conference on Computer Vision
  and Pattern Recognition}, pp.\  24174--24184, 2024{\natexlab{b}}.

\bibitem[Kim et~al.(2024)Kim, Lai, Liao, Murata, Takida, Uesaka, He, Mitsufuji,
  and Ermon]{kim2023consistency}
Kim, D., Lai, C.-H., Liao, W.-H., Murata, N., Takida, Y., Uesaka, T., He, Y.,
  Mitsufuji, Y., and Ermon, S.
\newblock Consistency trajectory models: Learning probability flow ode
  trajectory of diffusion.
\newblock In \emph{The Twelfth International Conference on Learning
  Representations}, 2024.

\bibitem[Kingma et~al.(2021)Kingma, Salimans, Poole, and
  Ho]{kingma2021variational}
Kingma, D., Salimans, T., Poole, B., and Ho, J.
\newblock Variational diffusion models.
\newblock \emph{Advances in neural information processing systems},
  34:\penalty0 21696--21707, 2021.

\bibitem[Kingma(2013)]{kingma2013auto}
Kingma, D.~P.
\newblock Auto-encoding variational bayes.
\newblock \emph{arXiv preprint arXiv:1312.6114}, 2013.

\bibitem[Kingma \& Gao(2023)Kingma and Gao]{kingma2023understanding}
Kingma, D.~P. and Gao, R.
\newblock Understanding diffusion objectives as the {ELBO} with simple data
  augmentation.
\newblock In \emph{Thirty-seventh Conference on Neural Information Processing
  Systems}, 2023.
\newblock URL \url{https://openreview.net/forum?id=NnMEadcdyD}.

\bibitem[Langley(2000)]{langley00}
Langley, P.
\newblock Crafting papers on machine learning.
\newblock In Langley, P. (ed.), \emph{Proceedings of the 17th International
  Conference on Machine Learning (ICML 2000)}, pp.\  1207--1216, Stanford, CA,
  2000. Morgan Kaufmann.

\bibitem[Li \& Zeng(2026)Li and
  Zeng]{li2026mitigatingcontractivitytrapdiffusion}
Li, S. and Zeng, D.
\newblock Mitigating the contractivity trap in diffusion odes via stein
  stabilization, 2026.
\newblock URL \url{https://arxiv.org/abs/2606.07835}.

\bibitem[Li et~al.(2023)Li, Chen, and Zeng]{li2023scire}
Li, S., Chen, W., and Zeng, D.
\newblock Scire-solver: Accelerating diffusion models sampling by
  score-integrand solver with recursive difference.
\newblock \emph{arXiv preprint arXiv:2308.07896}, 2023.

\bibitem[Li et~al.(2025)Li, Chen, and Zeng]{li2025evodiff}
Li, S., Chen, W., and Zeng, D.
\newblock {EVOD}iff: Entropy-aware variance optimized diffusion inference.
\newblock In \emph{Advances in Neural Information Processing Systems},
  volume~38, pp.\  148134--148181, 2025.
\newblock URL \url{https://openreview.net/forum?id=rKASv92Myl}.

\bibitem[Li et~al.(2024)Li, Zhou, Yu, Song, and Yang]{li2024coupled}
Li, W., Zhou, H., Yu, J., Song, Z., and Yang, W.
\newblock Coupled mamba: Enhanced multimodal fusion with coupled state space
  model.
\newblock In \emph{The Thirty-eighth Annual Conference on Neural Information
  Processing Systems}, 2024.
\newblock URL \url{https://openreview.net/forum?id=UXEo3uNNIX}.

\bibitem[Li et~al.(2026)Li, Song, Zhou, Yu, Zhang, and Yang]{li2026loramixer}
Li, W., Song, Z., Zhou, H., Yu, J., Zhang, Y., and Yang, W.
\newblock Lo{RA}-mixer: Coordinate modular lo{RA} experts through serial
  attention routing.
\newblock In \emph{The Fourteenth International Conference on Learning
  Representations}, 2026.
\newblock URL \url{https://openreview.net/forum?id=GMP1S4R6Ke}.

\bibitem[Lin et~al.(2024)Lin, Liu, Li, and Yang]{lin2024common}
Lin, S., Liu, B., Li, J., and Yang, X.
\newblock Common diffusion noise schedules and sample steps are flawed.
\newblock In \emph{Proceedings of the IEEE/CVF winter conference on
  applications of computer vision}, pp.\  5404--5411, 2024.

\bibitem[Lipman et~al.(2023)Lipman, Chen, Ben-Hamu, Nickel, and
  Le]{lipman2023flow}
Lipman, Y., Chen, R. T.~Q., Ben-Hamu, H., Nickel, M., and Le, M.
\newblock Flow matching for generative modeling.
\newblock In \emph{The Eleventh International Conference on Learning
  Representations}, 2023.
\newblock URL \url{https://openreview.net/forum?id=PqvMRDCJT9t}.

\bibitem[Liu et~al.(2022)Liu, Ren, Lin, and Zhao]{liu2022pseudo}
Liu, L., Ren, Y., Lin, Z., and Zhao, Z.
\newblock Pseudo numerical methods for diffusion models on manifolds.
\newblock In \emph{International Conference on Learning Representations}, 2022.
\newblock URL \url{https://openreview.net/forum?id=PlKWVd2yBkY}.

\bibitem[Liu et~al.(2023{\natexlab{a}})Liu, Gong, and qiang liu]{liu2023flow}
Liu, X., Gong, C., and qiang liu.
\newblock Flow straight and fast: Learning to generate and transfer data with
  rectified flow.
\newblock In \emph{The Eleventh International Conference on Learning
  Representations}, 2023{\natexlab{a}}.
\newblock URL \url{https://openreview.net/forum?id=XVjTT1nw5z}.

\bibitem[Liu et~al.(2023{\natexlab{b}})Liu, Park, Azadi, Zhang, Chopikyan, Hu,
  Shi, Rohrbach, and Darrell]{liu2023more}
Liu, X., Park, D.~H., Azadi, S., Zhang, G., Chopikyan, A., Hu, Y., Shi, H.,
  Rohrbach, A., and Darrell, T.
\newblock More control for free! image synthesis with semantic diffusion
  guidance.
\newblock In \emph{Proceedings of the IEEE/CVF winter conference on
  applications of computer vision}, pp.\  289--299, 2023{\natexlab{b}}.

\bibitem[Liu et~al.(2024)Liu, Zhang, Ma, Peng, and qiang liu]{liu2024instaflow}
Liu, X., Zhang, X., Ma, J., Peng, J., and qiang liu.
\newblock Instaflow: One step is enough for high-quality diffusion-based
  text-to-image generation.
\newblock In \emph{The Twelfth International Conference on Learning
  Representations}, 2024.
\newblock URL \url{https://openreview.net/forum?id=1k4yZbbDqX}.

\bibitem[Lu et~al.(2022)Lu, Zhou, Bao, Chen, Li, and Zhu]{lu2022dpm}
Lu, C., Zhou, Y., Bao, F., Chen, J., Li, C., and Zhu, J.
\newblock Dpm-solver: A fast ode solver for diffusion probabilistic model
  sampling in around 10 steps.
\newblock \emph{Advances in Neural Information Processing Systems},
  35:\penalty0 5775--5787, 2022.

\bibitem[Lu et~al.(2025)Lu, Zhou, Bao, Chen, Li, and Zhu]{lu2025dpm}
Lu, C., Zhou, Y., Bao, F., Chen, J., Li, C., and Zhu, J.
\newblock Dpm-solver++: Fast solver for guided sampling of diffusion
  probabilistic models.
\newblock \emph{Machine Intelligence Research}, pp.\  1--22, 2025.

\bibitem[Luo et~al.(2023)Luo, Tan, Huang, Li, and Zhao]{luo2023latent}
Luo, S., Tan, Y., Huang, L., Li, J., and Zhao, H.
\newblock Latent consistency models: Synthesizing high-resolution images with
  few-step inference.
\newblock \emph{arXiv preprint arXiv:2310.04378}, 2023.

\bibitem[Luo et~al.(2024)Luo, Huang, Geng, Kolter, and Qi]{luo2024one}
Luo, W., Huang, Z., Geng, Z., Kolter, J.~Z., and Qi, G.-j.
\newblock One-step diffusion distillation through score implicit matching.
\newblock \emph{Advances in Neural Information Processing Systems},
  37:\penalty0 115377--115408, 2024.

\bibitem[Ma et~al.(2024{\natexlab{a}})Ma, Goldstein, Albergo, Boffi,
  Vanden-Eijnden, and Xie]{ma2024sit}
Ma, N., Goldstein, M., Albergo, M.~S., Boffi, N.~M., Vanden-Eijnden, E., and
  Xie, S.
\newblock Sit: Exploring flow and diffusion-based generative models with
  scalable interpolant transformers.
\newblock In \emph{European Conference on Computer Vision}, pp.\  23--40.
  Springer, 2024{\natexlab{a}}.

\bibitem[Ma et~al.(2025)Ma, Tong, Jia, Hu, Su, Zhang, Yang, Li, Jaakkola, Jia,
  et~al.]{ma2025inference}
Ma, N., Tong, S., Jia, H., Hu, H., Su, Y.-C., Zhang, M., Yang, X., Li, Y.,
  Jaakkola, T., Jia, X., et~al.
\newblock Inference-time scaling for diffusion models beyond scaling denoising
  steps.
\newblock \emph{arXiv preprint arXiv:2501.09732}, 2025.

\bibitem[Ma et~al.(2024{\natexlab{b}})Ma, Fang, and Wang]{ma2024deepcache}
Ma, X., Fang, G., and Wang, X.
\newblock Deepcache: Accelerating diffusion models for free.
\newblock In \emph{Proceedings of the IEEE/CVF conference on computer vision
  and pattern recognition}, pp.\  15762--15772, 2024{\natexlab{b}}.

\bibitem[Meng et~al.(2023)Meng, Rombach, Gao, Kingma, Ermon, Ho, and
  Salimans]{meng2023distillation}
Meng, C., Rombach, R., Gao, R., Kingma, D., Ermon, S., Ho, J., and Salimans, T.
\newblock On distillation of guided diffusion models.
\newblock In \emph{Proceedings of the IEEE/CVF Conference on Computer Vision
  and Pattern Recognition (CVPR)}, pp.\  14297--14306, June 2023.

\bibitem[Nichol \& Dhariwal(2021)Nichol and Dhariwal]{nichol2021improved}
Nichol, A.~Q. and Dhariwal, P.
\newblock Improved denoising diffusion probabilistic models.
\newblock In \emph{International Conference on Machine Learning}, pp.\
  8162--8171. PMLR, 2021.

\bibitem[Ning et~al.(2024)Ning, Li, Su, Salah, and
  Ertugrul]{ning2024elucidating}
Ning, M., Li, M., Su, J., Salah, A.~A., and Ertugrul, I.~O.
\newblock Elucidating the exposure bias in diffusion models.
\newblock In \emph{The Twelfth International Conference on Learning
  Representations}, 2024.
\newblock URL \url{https://openreview.net/forum?id=xEJMoj1SpX}.

\bibitem[Park et~al.(2025)Park, Jung, Bae, and Yun]{park2025temporal}
Park, Y., Jung, H., Bae, S., and Yun, S.-Y.
\newblock Temporal alignment guidance: On-manifold sampling in diffusion
  models.
\newblock In \emph{NeurIPS 2025 Workshop on Structured Probabilistic Inference
  {\&} Generative Modeling}, 2025.
\newblock URL \url{https://openreview.net/forum?id=TubfvqZNEr}.

\bibitem[Podell et~al.(2024)Podell, English, Lacey, Blattmann, Dockhorn,
  M{\"u}ller, Penna, and Rombach]{podell2024sdxl}
Podell, D., English, Z., Lacey, K., Blattmann, A., Dockhorn, T., M{\"u}ller,
  J., Penna, J., and Rombach, R.
\newblock {SDXL}: Improving latent diffusion models for high-resolution image
  synthesis.
\newblock In \emph{The Twelfth International Conference on Learning
  Representations}, 2024.
\newblock URL \url{https://openreview.net/forum?id=di52zR8xgf}.

\bibitem[Ren et~al.(2019)Ren, Zhao, and Ermon]{ren2019adaptive}
Ren, H., Zhao, S., and Ermon, S.
\newblock Adaptive antithetic sampling for variance reduction.
\newblock In \emph{International Conference on Machine Learning}, pp.\
  5420--5428. PMLR, 2019.

\bibitem[Rombach et~al.(2022)Rombach, Blattmann, Lorenz, Esser, and
  Ommer]{rombach2022high}
Rombach, R., Blattmann, A., Lorenz, D., Esser, P., and Ommer, B.
\newblock High-resolution image synthesis with latent diffusion models.
\newblock In \emph{Proceedings of the IEEE/CVF conference on computer vision
  and pattern recognition}, pp.\  10684--10695, 2022.

\bibitem[Sabour et~al.(2024)Sabour, Fidler, and Kreis]{sabour2024align}
Sabour, A., Fidler, S., and Kreis, K.
\newblock Align your steps: Optimizing sampling schedules in diffusion models.
\newblock In \emph{Forty-first International Conference on Machine Learning},
  2024.
\newblock URL \url{https://openreview.net/forum?id=nBGBzV4It3}.

\bibitem[Sabour et~al.(2025)Sabour, Fidler, and Kreis]{sabour2025align}
Sabour, A., Fidler, S., and Kreis, K.
\newblock Align your flow: Scaling continuous-time flow map distillation.
\newblock In \emph{The Thirty-ninth Annual Conference on Neural Information
  Processing Systems}, 2025.
\newblock URL \url{https://openreview.net/forum?id=pzHuesCvcO}.

\bibitem[Salimans \& Ho(2022)Salimans and Ho]{salimans2022progressive}
Salimans, T. and Ho, J.
\newblock Progressive distillation for fast sampling of diffusion models.
\newblock In \emph{International Conference on Learning Representations}, 2022.
\newblock URL \url{https://openreview.net/forum?id=TIdIXIpzhoI}.

\bibitem[Salimans et~al.(2016)Salimans, Goodfellow, Zaremba, Cheung, Radford,
  and Chen]{salimans2016improved}
Salimans, T., Goodfellow, I., Zaremba, W., Cheung, V., Radford, A., and Chen,
  X.
\newblock Improved techniques for training gans.
\newblock \emph{Advances in neural information processing systems}, 29, 2016.

\bibitem[Sauer et~al.(2022)Sauer, Schwarz, and Geiger]{sauer2022stylegan}
Sauer, A., Schwarz, K., and Geiger, A.
\newblock Stylegan-xl: Scaling stylegan to large diverse datasets.
\newblock In \emph{ACM SIGGRAPH 2022 conference proceedings}, pp.\  1--10,
  2022.

\bibitem[Sauer et~al.(2024)Sauer, Lorenz, Blattmann, and
  Rombach]{sauer2024adversarial}
Sauer, A., Lorenz, D., Blattmann, A., and Rombach, R.
\newblock Adversarial diffusion distillation.
\newblock In \emph{European Conference on Computer Vision}, pp.\  87--103.
  Springer, 2024.

\bibitem[Shen et~al.(2024)Shen, Song, Xue, Wang, and Liu]{Shen_2024_CVPR}
Shen, D., Song, G., Xue, Z., Wang, F.-Y., and Liu, Y.
\newblock Rethinking the spatial inconsistency in classifier-free diffusion
  guidance.
\newblock In \emph{Proceedings of the IEEE/CVF Conference on Computer Vision
  and Pattern Recognition (CVPR)}, pp.\  9370--9379, June 2024.

\bibitem[Shen et~al.(2025)Shen, GAN, and Ling]{shen2025information}
Shen, Y., GAN, L., and Ling, C.
\newblock Information theoretic learning for diffusion models with warm start.
\newblock In \emph{The Thirty-ninth Annual Conference on Neural Information
  Processing Systems}, 2025.
\newblock URL \url{https://openreview.net/forum?id=3IbKbmNci3}.

\bibitem[Shih et~al.(2023)Shih, Belkhale, Ermon, Sadigh, and
  Anari]{shih2023parallel}
Shih, A., Belkhale, S., Ermon, S., Sadigh, D., and Anari, N.
\newblock Parallel sampling of diffusion models.
\newblock \emph{Advances in Neural Information Processing Systems},
  36:\penalty0 4263--4276, 2023.

\bibitem[Sohl-Dickstein et~al.(2015)Sohl-Dickstein, Weiss, Maheswaranathan, and
  Ganguli]{sohl2015deep}
Sohl-Dickstein, J., Weiss, E., Maheswaranathan, N., and Ganguli, S.
\newblock Deep unsupervised learning using nonequilibrium thermodynamics.
\newblock In \emph{International conference on machine learning}, pp.\
  2256--2265. PMLR, 2015.

\bibitem[Song et~al.(2021{\natexlab{a}})Song, Meng, and
  Ermon]{song2021denoising}
Song, J., Meng, C., and Ermon, S.
\newblock Denoising diffusion implicit models.
\newblock In \emph{International Conference on Learning Representations},
  2021{\natexlab{a}}.
\newblock URL \url{https://openreview.net/forum?id=St1giarCHLP}.

\bibitem[Song et~al.(2021{\natexlab{b}})Song, Sohl-Dickstein, Kingma, Kumar,
  Ermon, and Poole]{song2020score}
Song, Y., Sohl-Dickstein, J., Kingma, D.~P., Kumar, A., Ermon, S., and Poole,
  B.
\newblock Score-based generative modeling through stochastic differential
  equations.
\newblock In \emph{International Conference on Learning Representations},
  2021{\natexlab{b}}.

\bibitem[Song et~al.(2021{\natexlab{c}})Song, Sohl-Dickstein, Kingma, Kumar,
  Ermon, and Poole]{song2021score}
Song, Y., Sohl-Dickstein, J., Kingma, D.~P., Kumar, A., Ermon, S., and Poole,
  B.
\newblock Score-based generative modeling through stochastic differential
  equations.
\newblock In \emph{International Conference on Learning Representations},
  2021{\natexlab{c}}.
\newblock URL \url{https://openreview.net/forum?id=PxTIG12RRHS}.

\bibitem[Song et~al.(2023)Song, Dhariwal, Chen, and
  Sutskever]{song2023consistency}
Song, Y., Dhariwal, P., Chen, M., and Sutskever, I.
\newblock Consistency models.
\newblock In \emph{International Conference on Machine Learning}, pp.\
  32211--32252. PMLR, 2023.

\bibitem[Stein et~al.(2023)Stein, Cresswell, Hosseinzadeh, Sui, Ross,
  Villecroze, Liu, Caterini, Taylor, and Loaiza-Ganem]{stein2023exposing}
Stein, G., Cresswell, J., Hosseinzadeh, R., Sui, Y., Ross, B., Villecroze, V.,
  Liu, Z., Caterini, A.~L., Taylor, E., and Loaiza-Ganem, G.
\newblock Exposing flaws of generative model evaluation metrics and their
  unfair treatment of diffusion models.
\newblock \emph{Advances in Neural Information Processing Systems},
  36:\penalty0 3732--3784, 2023.

\bibitem[Tang et~al.(2025)Tang, Peng, Tang, Hong, Wang, and
  Chang]{tang2025inferencetime}
Tang, Z., Peng, J., Tang, J., Hong, M., Wang, F., and Chang, T.-H.
\newblock Inference-time alignment of diffusion models with direct noise
  optimization.
\newblock In \emph{Forty-second International Conference on Machine Learning},
  2025.
\newblock URL \url{https://openreview.net/forum?id=JpbqiD7n9r}.

\bibitem[Tong et~al.(2025)Tong, Hoang, Liu, den Broeck, and
  Niepert]{tong2025learning}
Tong, V., Hoang, D.~T., Liu, A., den Broeck, G.~V., and Niepert, M.
\newblock Learning to discretize denoising diffusion {ODE}s.
\newblock In \emph{The Thirteenth International Conference on Learning
  Representations}, 2025.
\newblock URL \url{https://openreview.net/forum?id=xDrFWUmCne}.

\bibitem[Vahdat et~al.(2021)Vahdat, Kreis, and Kautz]{vahdat2021score}
Vahdat, A., Kreis, K., and Kautz, J.
\newblock Score-based generative modeling in latent space.
\newblock In \emph{Advances in Neural Information Processing Systems},
  volume~34, pp.\  11287--11302, 2021.

\bibitem[Wang et~al.(2025)Wang, Yang, Huang, Wang, and Li]{wang2025rectified}
Wang, F.-Y., Yang, L., Huang, Z., Wang, M., and Li, H.
\newblock Rectified diffusion: Straightness is not your need in rectified flow.
\newblock In \emph{The Thirteenth International Conference on Learning
  Representations}, 2025.
\newblock URL \url{https://openreview.net/forum?id=nEDToD1R8M}.

\bibitem[Wimbauer et~al.(2024)Wimbauer, Wu, Schoenfeld, Dai, Hou, He,
  Sanakoyeu, Zhang, Tsai, Kohler, et~al.]{wimbauer2024cache}
Wimbauer, F., Wu, B., Schoenfeld, E., Dai, X., Hou, J., He, Z., Sanakoyeu, A.,
  Zhang, P., Tsai, S., Kohler, J., et~al.
\newblock Cache me if you can: Accelerating diffusion models through block
  caching.
\newblock In \emph{Proceedings of the IEEE/CVF Conference on Computer Vision
  and Pattern Recognition}, pp.\  6211--6220, 2024.

\bibitem[Wu et~al.(2023)Wu, Zhou, Kawaguchi, and Zhang]{wu2023fast}
Wu, Z., Zhou, P., Kawaguchi, K., and Zhang, H.
\newblock Fast diffusion model.
\newblock \emph{arXiv preprint arXiv:2306.06991}, 2023.

\bibitem[Xu et~al.(2024{\natexlab{a}})Xu, Zeng, and Paisley]{xu2024sparse}
Xu, J., Zeng, D., and Paisley, J.
\newblock Sparse inducing points in deep gaussian processes: Enhancing modeling
  with denoising diffusion variational inference.
\newblock In \emph{Forty-first International Conference on Machine Learning},
  2024{\natexlab{a}}.
\newblock URL \url{https://openreview.net/forum?id=jTn4AIOgpM}.

\bibitem[Xu et~al.(2024{\natexlab{b}})Xu, Corso, Jaakkola, Vahdat, and
  Kreis]{xu2024disco}
Xu, Y., Corso, G., Jaakkola, T., Vahdat, A., and Kreis, K.
\newblock Disco-diff: Enhancing continuous diffusion models with discrete
  latents.
\newblock \emph{arXiv preprint arXiv:2407.03300}, 2024{\natexlab{b}}.

\bibitem[Yang et~al.(2024)Yang, Ding, Cai, Yu, Wang, and Shi]{yang2024guidance}
Yang, L., Ding, S., Cai, Y., Yu, J., Wang, J., and Shi, Y.
\newblock Guidance with spherical gaussian constraint for conditional
  diffusion.
\newblock In \emph{Forty-first International Conference on Machine Learning},
  2024.
\newblock URL \url{https://openreview.net/forum?id=VtqyurB4Af}.

\bibitem[Ye et~al.(2024)Ye, Lin, Han, Xu, Liu, Liang, Ma, Zou, and
  Ermon]{ye2024tfg}
Ye, H., Lin, H., Han, J., Xu, M., Liu, S., Liang, Y., Ma, J., Zou, J.~Y., and
  Ermon, S.
\newblock Tfg: Unified training-free guidance for diffusion models.
\newblock \emph{Advances in Neural Information Processing Systems},
  37:\penalty0 22370--22417, 2024.

\bibitem[Yin et~al.(2024)Yin, Gharbi, Zhang, Shechtman, Durand, Freeman, and
  Park]{yin2024one}
Yin, T., Gharbi, M., Zhang, R., Shechtman, E., Durand, F., Freeman, W.~T., and
  Park, T.
\newblock One-step diffusion with distribution matching distillation.
\newblock In \emph{Proceedings of the IEEE/CVF conference on computer vision
  and pattern recognition}, pp.\  6613--6623, 2024.

\bibitem[Zhang et~al.(2025)Zhang, Liu, Park, Zhang, and
  Xu]{zhang2025antiexposure}
Zhang, J., Liu, D., Park, E., Zhang, S., and Xu, C.
\newblock Anti-exposure bias in diffusion models.
\newblock In \emph{The Thirteenth International Conference on Learning
  Representations}, 2025.
\newblock URL \url{https://openreview.net/forum?id=MtDd7rWok1}.

\bibitem[Zhang \& Chen(2023)Zhang and Chen]{zhang2023fast}
Zhang, Q. and Chen, Y.
\newblock Fast sampling of diffusion models with exponential integrator.
\newblock In \emph{The Eleventh International Conference on Learning
  Representations}, 2023.
\newblock URL \url{https://openreview.net/forum?id=Loek7hfb46P}.

\bibitem[Zhao et~al.(2023)Zhao, Bai, Rao, Zhou, and Lu]{zhao2023unipc}
Zhao, W., Bai, L., Rao, Y., Zhou, J., and Lu, J.
\newblock Uni{PC}: A unified predictor-corrector framework for fast sampling of
  diffusion models.
\newblock In \emph{Thirty-seventh Conference on Neural Information Processing
  Systems}, 2023.
\newblock URL \url{https://openreview.net/forum?id=hrkmlPhp1u}.

\bibitem[Zheng et~al.(2023)Zheng, Lu, Chen, and Zhu]{zheng2023dpm}
Zheng, K., Lu, C., Chen, J., and Zhu, J.
\newblock Dpm-solver-v3: Improved diffusion ode solver with empirical model
  statistics.
\newblock \emph{Advances in Neural Information Processing Systems},
  36:\penalty0 55502--55542, 2023.

\bibitem[Zhou et~al.(2024)Zhou, Zheng, Gu, Wang, and Huang]{zhou2024sid2a}
Zhou, M., Zheng, H., Gu, Y., Wang, Z., and Huang, H.
\newblock Adversarial score identity distillation: Rapidly surpassing the
  teacher in one step.
\newblock \emph{arXiv preprint arXiv:2410.14919}, 2024.

\end{thebibliography}
\bibliographystyle{icml2026}

\newpage
\appendix
\onecolumn
\begin{center}
\LARGE\textbf{Appendix}
\end{center} 
\section{Proofs}

\subsection{Proof of Proposition~\ref{prop:variance}}
\label{sec:proof_variance}
\begin{proof}
Under the ideal forward-aligned observation model, each
observation is given by
\begin{equation}
\boldsymbol{y}_{i}
=
\boldsymbol{x}_{0,\star}
+
\boldsymbol{\eta}_{i},
\qquad
\operatorname{Cov}(\boldsymbol{\eta}_{i})
=
R_{i}\boldsymbol{I},
\qquad
R_{i}
=
\frac{c}{\operatorname{SNR}(t_{i})}.
\label{eq:appendix_blue_observation}
\end{equation}
We consider a linear unbiased estimator
\begin{equation}
\hat{\boldsymbol{x}}_{0,\operatorname{ideal}}
=
\sum_{i=1}^{n}
w_{i}\boldsymbol{y}_{i},
\qquad
\sum_{i=1}^{n}w_{i}=1.
\label{eq:appendix_blue_linear_estimator}
\end{equation}
Because the observation errors are mutually uncorrelated,
the covariance of the estimator is
\begin{equation}
\operatorname{Cov}
\left(
\hat{\boldsymbol{x}}_{0,\operatorname{ideal}}
\right)
=
\sum_{i=1}^{n}
w_{i}^{2}R_{i}\boldsymbol{I}.
\label{eq:appendix_blue_covariance}
\end{equation}
Minimizing the scalar coefficient
$\sum_{i=1}^{n}w_{i}^{2}R_{i}$ subject to
$\sum_{i=1}^{n}w_{i}=1$ gives the Lagrangian
\begin{equation}
\operatorname{L}
=
\sum_{i=1}^{n}
w_{i}^{2}R_{i}
-
\lambda
\left(
\sum_{i=1}^{n}w_{i}-1
\right).
\label{eq:appendix_blue_lagrangian}
\end{equation}
Taking derivatives with respect to $w_{i}$ yields
\begin{equation}
\frac{\partial \operatorname{L}}{\partial w_{i}}
=
2w_{i}R_{i}
-
\lambda
=
0,
\label{eq:appendix_blue_stationary_condition}
\end{equation}
and therefore
\begin{equation}
w_{i}
=
\frac{R_{i}^{-1}}
{\sum_{j=1}^{n}R_{j}^{-1}}
=
\frac{\operatorname{SNR}(t_{i})}
{\sum_{j=1}^{n}\operatorname{SNR}(t_{j})}.
\label{eq:appendix_blue_optimal_weights}
\end{equation}
Substituting these weights into
Eq.~\eqref{eq:appendix_blue_covariance} gives
\begin{equation}
\operatorname{Cov}
\left(
\hat{\boldsymbol{x}}_{0,\operatorname{ideal}}^{\star}
\right)
=
\left(
\sum_{i=1}^{n}
R_{i}^{-1}
\right)^{-1}
\boldsymbol{I}
=
\frac{c}
{\sum_{i=1}^{n}
\operatorname{SNR}(t_{i})}
\boldsymbol{I}.
\label{eq:appendix_blue_final_covariance}
\end{equation}
Since
$\sum_{i=1}^{n}R_{i}^{-1}>R_{j}^{-1}$
for every $j$ whenever $n\geq 2$, it follows that
\begin{equation}
\operatorname{Cov}
\left(
\hat{\boldsymbol{x}}_{0,\operatorname{ideal}}^{\star}
\right)
\prec
R_{j}\boldsymbol{I},
\qquad
j=1,\ldots,n.
\label{eq:appendix_blue_strict_reduction}
\end{equation}
Thus, the ideal precision-weighted fused estimator has
strictly lower covariance than any individual observation.

This result characterizes the ideal forward-aligned
observation model. In DiFA, practical denoiser predictions
are organized according to the resulting anchor-consistency
and reliability-ordering principles, rather than being
assumed to satisfy the same covariance model exactly. 

The proof is complete.
\end{proof}

\subsection{Proof of Theorem~\ref{thm:kalman_equivalence}}
\label{sec:proof_kalman}
\begin{proof}
For the static anchor model, let
$\hat{\boldsymbol{x}}_{0,i}^{\operatorname{rec}}$
denote the recursive estimate after processing the first
$i$ observations, and let $p_i\boldsymbol{I}$ denote its
estimation covariance. We initialize the recursion using the
first observation:
\begin{equation}
\hat{\boldsymbol{x}}_{0,1}^{\operatorname{rec}} = \boldsymbol{y}_1, \qquad p_1 = R_1.
\end{equation}
For $i=2,\ldots,n$, the static-state Kalman gain is
\begin{equation}
k_i = \frac{ p_{i-1} }{ p_{i-1}+R_i },
\end{equation}
and the recursive update is
\begin{equation}
\hat{\boldsymbol{x}}_{0,i}^{\operatorname{rec}} = \hat{\boldsymbol{x}}_{0,i-1}^{\operatorname{rec}} + k_i \left( \boldsymbol{y}_i - \hat{\boldsymbol{x}}_{0,i-1}^{\operatorname{rec}} \right).
\end{equation}
The corresponding covariance update is
\begin{equation}
p_i = (1-k_i)p_{i-1} = \frac{ p_{i-1}R_i }{ p_{i-1}+R_i },
\end{equation}
which implies
\begin{equation}
p_i^{-1} = p_{i-1}^{-1} + R_i^{-1}.
\end{equation}
Since $p_1=R_1$, repeated application yields
\begin{equation}
p_n^{-1} = \sum_{i=1}^{n} R_i^{-1}.
\label{eq:recursive_precision_sum}
\end{equation}

Next, the recursive mean update can be rewritten as
\begin{equation}
\hat{\boldsymbol{x}}_{0,i}^{\operatorname{rec}} = \frac{ R_i\hat{\boldsymbol{x}}_{0,i-1}^{\operatorname{rec}} + p_{i-1}\boldsymbol{y}_i }{ p_{i-1}+R_i }.
\end{equation}
Combining this relation with the covariance update gives
\begin{equation}
p_i^{-1} \hat{\boldsymbol{x}}_{0,i}^{\operatorname{rec}} = p_{i-1}^{-1} \hat{\boldsymbol{x}}_{0,i-1}^{\operatorname{rec}} + R_i^{-1}\boldsymbol{y}_i.
\end{equation}
Using
\begin{equation}
p_1^{-1} \hat{\boldsymbol{x}}_{0,1}^{\operatorname{rec}} = R_1^{-1}\boldsymbol{y}_1,
\end{equation}
we obtain
\begin{equation}
p_n^{-1} \hat{\boldsymbol{x}}_{0,n}^{\operatorname{rec}} = \sum_{i=1}^{n} R_i^{-1}\boldsymbol{y}_i.
\label{eq:recursive_information_sum}
\end{equation}

Substituting Eq.~\eqref{eq:recursive_precision_sum} into
Eq.~\eqref{eq:recursive_information_sum} yields
\begin{equation}
\hat{\boldsymbol{x}}_{0,n}^{\operatorname{rec}} = \frac{ \sum_{i=1}^{n} R_i^{-1}\boldsymbol{y}_i }{ \sum_{i=1}^{n} R_i^{-1} }.
\end{equation}
Finally, since
\begin{equation}
R_i^{-1} = \frac{ \operatorname{SNR}(t_i) }{ c },
\end{equation}
the common scale factor $c$ cancels out, giving
\begin{equation}
\hat{\boldsymbol{x}}_{0,n}^{\operatorname{rec}} = \frac{ \sum_{i=1}^{n} \operatorname{SNR}(t_i)\boldsymbol{y}_i }{ \sum_{i=1}^{n} \operatorname{SNR}(t_i) } = \hat{\boldsymbol{x}}_{0,\operatorname{ideal}}^{\star}.
\end{equation}
Therefore, under the ideal independent-view observation
model, the full-history precision-weighted anchor estimator
is exactly equivalent to its static-state Kalman recursive
realization.

The proof is complete. 
\end{proof}

\subsection{Lightweight Instantiation Used in Experiments}
\label{sec:implementation_details}

This section specifies the lightweight DiFA instantiation
used in the main pixel-space diffusion experiments. The main
text formulates DiFA through a causal historical reference
anchor and an anchor-relative deviation modulation operator.
Here, we provide the explicit forms of the alignment,
compatibility weighting, and deviation guidance modules, as
well as the default hyperparameter settings used to obtain
the main experimental results. Additional configurations used
for latent-diffusion and flow-matching validation are
specified separately in
Appendix~\ref{app:ablation}.

\paragraph{Causal History Buffer.}
At reverse step $t_i$, the history buffer stores at most $K$ previously processed clean-signal predictions and their corresponding logSNR values:
\begin{equation}
\operatorname{H}_{K}(t_i) = \left\{ \left( \hat{\boldsymbol{x}}_{0}^{(t_j)}, \ell_j \right) \right\}_{t_j\in\operatorname{W}_{K}(t_i)}, \qquad \ell_j = \log \operatorname{SNR}(t_j),
\end{equation}
where $\operatorname{W}_{K}(t_i) \subseteq \{t_j : j>i\}$ contains the most recent historical steps along the reverse trajectory. The instantaneous prediction $\hat{\boldsymbol{x}}_{0}^{(t_i)}$ is used only as a query for constructing the historical reference and is not included among the aggregated reference values.

\paragraph{Historical Prediction Alignment.}
To reduce local scale mismatch between the current prediction and historical candidates, we apply channel-wise affine mean--variance alignment. Let $\mu_c(\boldsymbol{x})$ and $\sigma_c(\boldsymbol{x})$ denote the spatial mean and standard deviation of channel $c$. The aligned historical prediction is defined by
\begin{equation}
\tilde{\boldsymbol{x}}_{0,c}^{(t_j)} = \frac{ \hat{\boldsymbol{x}}_{0,c}^{(t_j)} - \mu_c \left( \hat{\boldsymbol{x}}_{0}^{(t_j)} \right) }{ \sigma_c \left( \hat{\boldsymbol{x}}_{0}^{(t_j)} \right) + \epsilon_{\rm a} } \cdot \sigma_c \left( \hat{\boldsymbol{x}}_{0}^{(t_i)} \right) + \mu_c \left( \hat{\boldsymbol{x}}_{0}^{(t_i)} \right),
\label{eq:appendix_alignment}
\end{equation}
where $\epsilon_{\rm a}>0$ is a numerical stabilizer. This operation matches the channel-wise first- and second-order statistics of each historical prediction to those of the current prediction before reference aggregation.

\paragraph{Structural and Noise-Level Compatibility.}
We compute structural compatibility after lightweight spatial average pooling. Let $\operatorname{P}_{\rm sim}(\cdot)$ denote average pooling with kernel size $k_{\rm sim}\times k_{\rm sim}$. For each historical candidate, the local similarity score is
\begin{equation}
m_{ij} = \frac{ \left\langle \operatorname{vec} \left( \operatorname{P}_{\rm sim} \left( \hat{\boldsymbol{x}}_{0}^{(t_i)} \right) \right), \operatorname{vec} \left( \operatorname{P}_{\rm sim} \left( \tilde{\boldsymbol{x}}_{0}^{(t_j)} \right) \right) \right\rangle }{ \left\| \operatorname{vec} \left( \operatorname{P}_{\rm sim} \left( \hat{\boldsymbol{x}}_{0}^{(t_i)} \right) \right) \right\|_2 \left\| \operatorname{vec} \left( \operatorname{P}_{\rm sim} \left( \tilde{\boldsymbol{x}}_{0}^{(t_j)} \right) \right) \right\|_2 + \epsilon_{\rm s} },
\label{eq:appendix_similarity}
\end{equation}
where $\epsilon_{\rm s}>0$ prevents numerical instability. The similarity scores are normalized within the current history window:
\begin{equation}
\bar{m}_{ij} = \frac{ m_{ij} - \mu_{m,i} }{ \sigma_{m,i} + \epsilon_{\rm n} }, \qquad \mu_{m,i} = \frac{1}{|\operatorname{W}_{K}(t_i)|} \sum_{t_j\in\operatorname{W}_{K}(t_i)} m_{ij},
\label{eq:appendix_similarity_norm}
\end{equation}
where $\sigma_{m,i}$ is the standard deviation of the similarity scores within the current temporal window. The compatibility logit combines structural agreement and logSNR proximity:
\begin{equation}
q_{ij} = \tau \bar{m}_{ij} - \mu(\ell_i) \left| \ell_i-\ell_j \right|,
\label{eq:appendix_compatibility_logit}
\end{equation}
where $\tau$ controls the sharpness of structural selection and $\mu(\ell_i)$ controls the strength of noise-level compatibility at the current reverse step. The normalized historical weights are
\begin{equation}
a_{ij} = \frac{ \exp(q_{ij}) }{ \sum_{t_k\in\operatorname{W}_{K}(t_i)} \exp(q_{ik}) }.
\label{eq:appendix_compatibility_weights}
\end{equation}
The causal historical reference anchor is then
\begin{equation}
\hat{\boldsymbol{x}}_{0}^{\operatorname{cons}}(t_i) = \sum_{t_j\in\operatorname{W}_{K}(t_i)} a_{ij} \tilde{\boldsymbol{x}}_{0}^{(t_j)}.
\label{eq:appendix_consensus}
\end{equation}

\paragraph{Deviation Modulation.}
Given the historical reference anchor, we compute the anchor-relative deviation:
\begin{equation}
\boldsymbol{r}_{t_i} = \hat{\boldsymbol{x}}_{0}^{(t_i)} - \hat{\boldsymbol{x}}_{0}^{\operatorname{cons}}(t_i).
\end{equation}
To reduce residual corrections dominated by magnitude variation parallel to the current prediction, we apply an orthogonal residual projection:
\begin{equation}
\boldsymbol{r}_{t_i}^{\perp} = \boldsymbol{r}_{t_i} - \frac{ \left\langle \boldsymbol{r}_{t_i}, \hat{\boldsymbol{x}}_{0}^{(t_i)} \right\rangle }{ \left\| \hat{\boldsymbol{x}}_{0}^{(t_i)} \right\|_2^2 + \epsilon_{\rm p} } \hat{\boldsymbol{x}}_{0}^{(t_i)},
\label{eq:appendix_projection}
\end{equation}
where $\epsilon_{\rm p}>0$ is a numerical stabilizer. We then decompose the projected residual into locally smoothed and high-frequency components:
\begin{equation}
\boldsymbol{r}_{t_i}^{\operatorname{low}} = \operatorname{P}_{\rm res} \left( \boldsymbol{r}_{t_i}^{\perp} \right), \qquad \boldsymbol{r}_{t_i}^{\operatorname{high}} = \boldsymbol{r}_{t_i}^{\perp} - \boldsymbol{r}_{t_i}^{\operatorname{low}},
\label{eq:appendix_frequency_decomposition}
\end{equation}
where $\operatorname{P}_{\rm res}(\cdot)$ denotes local average pooling with kernel size $k_{\rm res}\times k_{\rm res}$. The high-frequency gate is defined as
\begin{equation}
\lambda_{\rm hf}(\ell_i) = \frac{1}{1+\exp(-\ell_i)}.
\label{eq:appendix_hf_gate}
\end{equation}
The deviation modulation operator used in our experiments is therefore instantiated as
\begin{equation}
\operatorname{G} \left( \boldsymbol{r}_{t_i}, \hat{\boldsymbol{x}}_{0}^{(t_i)}, \ell_i \right) = \boldsymbol{r}_{t_i}^{\operatorname{low}} + \lambda_{\rm hf}(\ell_i) \boldsymbol{r}_{t_i}^{\operatorname{high}}.
\label{eq:appendix_guidance_operator}
\end{equation}
Finally, the refined clean-signal prediction is
\begin{equation}
\hat{\boldsymbol{x}}_{0}^{\operatorname{DiFA}}(t_i) = \hat{\boldsymbol{x}}_{0}^{(t_i)} + \omega \operatorname{G} \left( \boldsymbol{r}_{t_i}, \hat{\boldsymbol{x}}_{0}^{(t_i)}, \ell_i \right).
\label{eq:appendix_final_update}
\end{equation}

\paragraph{Default Hyperparameters.}
Unless otherwise stated, all reported DiFA results use the lightweight instantiation above with the default settings in Table~\ref{tab:difa_default_hyperparameters}. Ablation experiments modify only the factors explicitly specified in their corresponding tables.

\begin{table}[htbp]
\centering
\caption{Default hyperparameter settings for the lightweight DiFA instantiation used in the main experiments.}
\label{tab:difa_default_hyperparameters}
\begin{tabular}{lll}
\toprule
Hyperparameter & Meaning & Default value \\
\midrule
$K$ & causal temporal window size & $3$ \\
$\tau$ & structural compatibility sharpness & $4.0$ \\
$s$ & reported DiFA deviation scale & $1.7$ \\
$\omega=s-1$ & residual guidance coefficient & $0.7$ \\
$k_{\rm sim}$ & similarity pooling kernel size & $3$ \\
$k_{\rm res}$ & residual smoothing kernel size & $5$ \\
$\lambda_{\rm hf}(\ell_i)$ & high-frequency gate & $\left(1+\exp(-\ell_i)\right)^{-1}$ \\
\bottomrule
\end{tabular}
\end{table}

The reported scale $s$ follows the convention that $s=1$ recovers the uncorrected baseline; therefore, the coefficient multiplying the deviation guidance is implemented as $\omega=s-1$. The function $\mu(\ell_i)$ and the numerical stabilizers $\epsilon_{\rm a}$, $\epsilon_{\rm s}$, $\epsilon_{\rm n}$, and $\epsilon_{\rm p}$ are set according to the implementation configuration used for each experimental setting.

\section{Additional Ablation and Validation Studies}
\label{app:ablation}

This appendix provides additional ablation, sensitivity, and validation studies that complement the main experiments and further clarify the contribution of each component in DiFA. Unless otherwise specified, all ablation results are obtained under the same pretrained model and solver setting as the corresponding baseline. Lower FID is better, and higher IS is better.

\begin{figure*}[t]
    \centering 
    \begin{subfigure}[b]{0.48\textwidth}
        \centering
        \includegraphics[width=\linewidth]{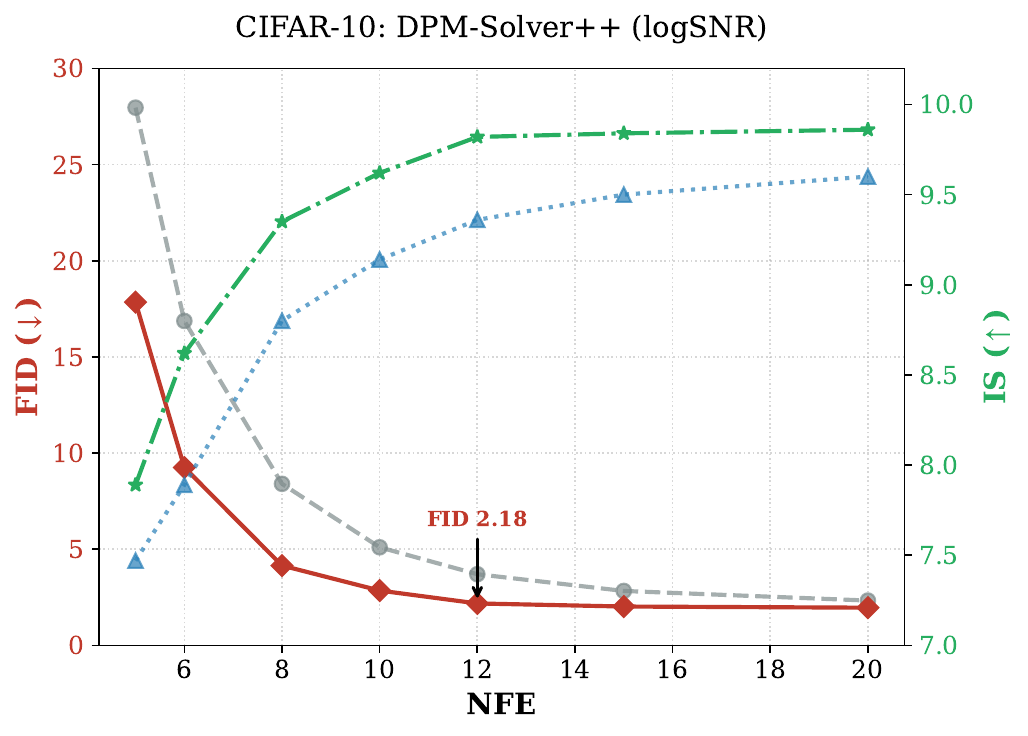}
        \caption{CIFAR-10: DPM-Solver++ (logSNR)}
        \label{fig:cifar_dpm}
    \end{subfigure}~~~~
    \hfill
    \begin{subfigure}[b]{0.48\textwidth}
        \centering
        \includegraphics[width=\linewidth]{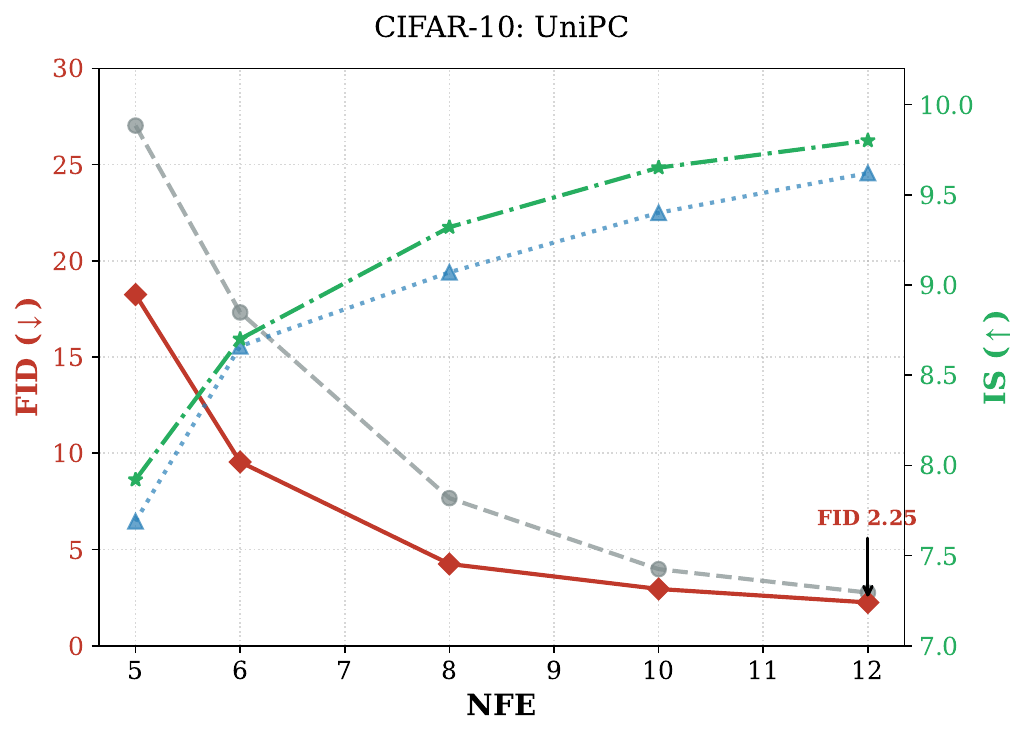}
        \caption{CIFAR-10: UniPC}
        \label{fig:cifar_unipc}
    \end{subfigure}
    
    \vspace{0.25cm}  
     
    \begin{subfigure}[b]{0.48\textwidth}
        \centering
        \includegraphics[width=\linewidth]{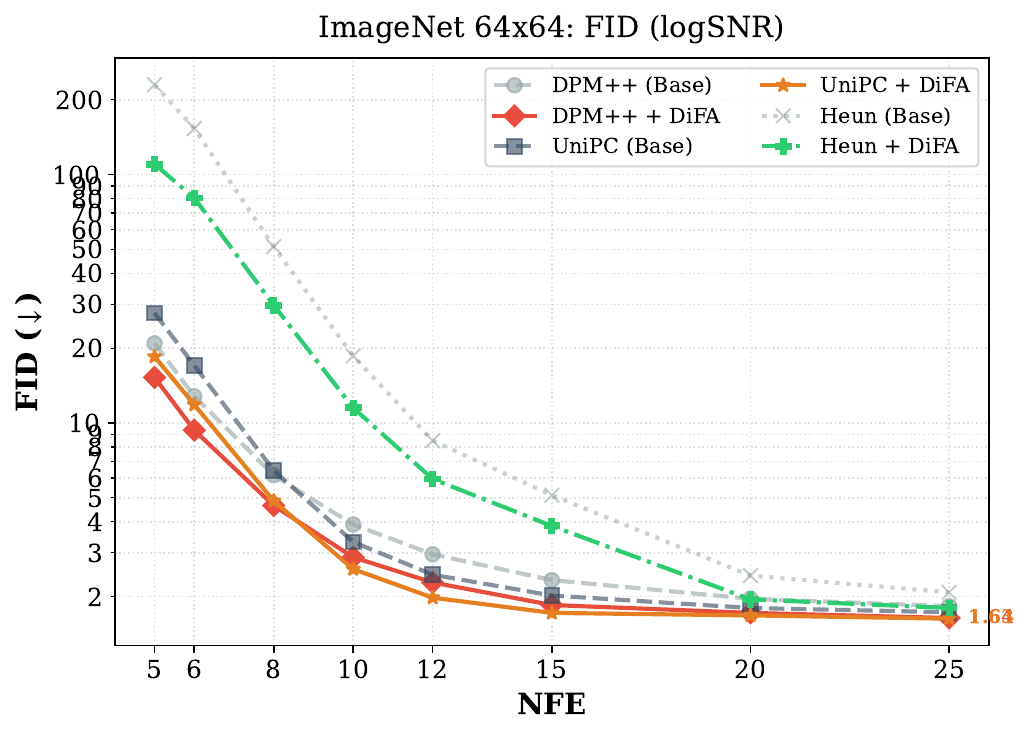}
        \caption{ImageNet: FID (logSNR)}
        \label{fig:img_log_fid}
    \end{subfigure}
    \hfill
    \begin{subfigure}[b]{0.48\textwidth}
        \centering
        \includegraphics[width=\linewidth]{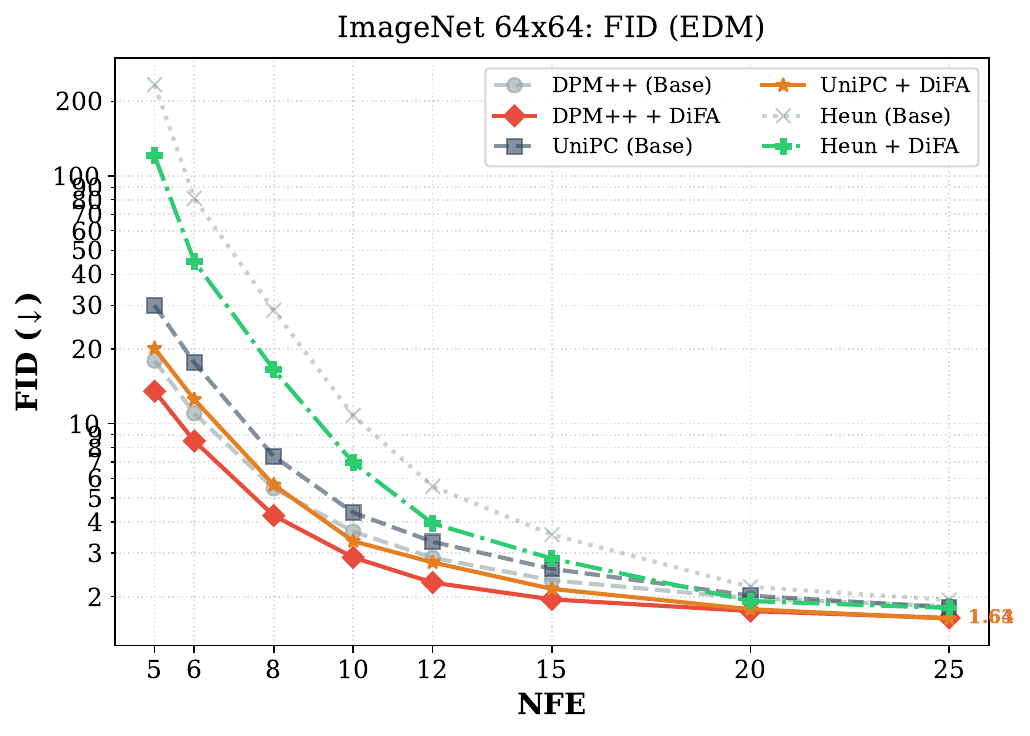}
        \caption{ImageNet: FID (EDM)}
        \label{fig:img_edm_fid}
    \end{subfigure}
    
    \caption{Quantitative Results. \textbf{Top Row:} On CIFAR-10, DiFA (solid red lines) significantly lowers FID compared to baselines (dashed gray lines) across all NFE regimes for both DPM-Solver++ and UniPC. Note the rapid convergence at NFE 12 (FID $\approx$ 2.2). \textbf{Bottom Row:} On ImageNet 64x64, DiFA consistently improves the convergence of various solvers (DPM++, UniPC, Heun) under both logSNR and EDM schedules, achieving strong training-free performance with an FID of $\approx$ 1.63 at 25 steps.}
    \label{fig:main_results}
\end{figure*}
\subsection{Ablation on Latent Diffusion Models}
\label{app:ldm_ablation}

We first evaluate DiFA on a latent diffusion model using the LSUN Bedroom dataset. For this experiment, we retain the notation used in
the corresponding implementation, where $\gamma_0$ and $\lambda$
denote the deviation strength and residual coefficient, respectively. 
The latent-diffusion experiments use a task-specific parameterization;
the corresponding symbols are defined in Table~\ref{tab:ldm_hyperparameter}. This experiment verifies that DiFA is not restricted to pixel-space diffusion models. Since DiFA operates on clean-signal predictions within the sampling loop, it can also be applied to latent-space samplers without changing the pretrained network.

\begin{table}[ht]
\centering
\caption{DiFA versus the naive ODE-solver baseline on LSUN Bedroom. The naive baseline corresponds to DPM-Solver++ with DiFA disabled.}
\label{tab:ldm_main}
\begin{tabular}{lccc}
\toprule
Method & 5 NFE & 10 NFE & 20 NFE \\
\midrule
Naive Baseline & 21.238 & 3.877 & 3.256 \\
DiFA (Full) & \textbf{5.912} & \textbf{3.579} & \textbf{2.939} \\
\midrule
FID Improvement & 72.2\% & 7.7\% & 9.7\% \\
\bottomrule
\end{tabular}
\end{table}

\begin{table}[ht]
\centering
\caption{Component ablation on LSUN Bedroom. We isolate the effects of historical consensus, SNR gating, and magnitude alignment.}
\label{tab:ldm_component_ablation}
\begin{tabular}{lcccccc}
\toprule
Configuration & $W$ & $SNR_{\rm lo}$ & $\phi$ & 5 NFE $\downarrow$ & 10 NFE $\downarrow$ & 20 NFE $\downarrow$ \\
\midrule
Naive Baseline & 1 & Full & 0.0 & 21.238 & 3.877 & 3.256 \\
w/o History & 1 & -4.5 & 0.5 & 21.238 & 3.877 & 3.256 \\
w/o SNR Gating & 3 & Full & 0.5 & \textbf{5.838} & 4.414 & 3.119 \\
w/o Energy ($\phi$) & 3 & -4.5 & 0.0 & 7.053 & 4.462 & 3.142 \\
DiFA (Full) & 3 & -4.5 & 0.5 & 5.912 & \textbf{3.579} & \textbf{2.939} \\
\bottomrule
\end{tabular}
\end{table}

The LDM ablation shows that historical consensus is essential: when the history window collapses to $W=1$, DiFA degenerates to the baseline. SNR gating is particularly useful at medium and high NFE budgets, where unguided activation may introduce unstable corrections. The magnitude alignment factor $\phi$ stabilizes latent-space energy and prevents excessive magnitude drift during deviation guidance.

\begin{table}[ht]
\centering
\caption{Hyperparameter robustness on LSUN Bedroom. A dash denotes the naive baseline without DiFA. All listed DiFA configurations improve over the naive baseline.}
\label{tab:ldm_hyperparameter}
\begin{tabular}{ccccccccc}
\toprule
$\gamma_0$ & $\lambda$ & $\phi$ & $SNR_{\rm lo}$ & $SNR_{\rm hi}$ & $\tau$ & 5 NFE $\downarrow$ & 10 NFE $\downarrow$ & 20 NFE $\downarrow$ \\
\midrule
-- & -- & -- & -- & -- & -- & 21.238 & 3.877 & 3.256 \\
1.40 & 0.5 & 0.5 & -2.5 & 3.0 & 0.8 & 9.142 & 3.684 & 3.230 \\
1.40 & 0.5 & 0.5 & -2.5 & 4.0 & 0.8 & 9.032 & 3.729 & 3.248 \\
1.50 & 0.5 & 0.5 & -4.5 & 3.0 & 0.8 & 6.535 & 3.658 & 3.104 \\
1.50 & 0.5 & 0.7 & -3.5 & 2.0 & 0.8 & 7.576 & \textbf{3.579} & 3.120 \\
1.75 & 0.4 & 0.5 & -4.5 & 5.0 & 0.8 & 6.069 & 4.465 & \textbf{2.939} \\
1.75 & 0.4 & 0.5 & -4.5 & 5.0 & 1.0 & 6.019 & 4.429 & 2.954 \\
\bottomrule
\end{tabular}
\end{table}

The robustness study shows that DiFA consistently improves over the naive baseline under different hyperparameter choices. The best configuration can vary across NFE budgets, suggesting that DiFA is not overly sensitive to a single parameter setting.

\subsection{Ablation on EDM}
\label{app:edm_ablation}

We further conduct controlled ablation experiments under the EDM setting at 10 NFE. These results isolate the contribution of the main DiFA components: historical consensus, adaptive SNR gating, and magnitude alignment.

\begin{table}[ht]
\centering
\caption{Main component ablation under the EDM setting at 10 NFE.}
\label{tab:edm_main_ablation}
\begin{tabular}{lccccc}
\toprule
Configuration & History ($W$) & SNR Gating ($SNR_{\rm lo}$) & Energy ($\phi$) & FID $\downarrow$ & IS $\uparrow$ \\
\midrule
Naive Baseline & 1 & Full & 0.0 & 5.071 & 9.139 \\
w/o History & 1 & -2.5 & 0.5 & 5.071 & 9.139 \\
w/o SNR Gating & 3 & Full & 0.5 & 3.375 & 9.574 \\
w/o Energy ($\phi$) & 3 & -2.5 & 0.0 & 2.846 & 9.661 \\
DiFA (Full) & 3 & -2.5 & 0.5 & \textbf{2.797} & \textbf{9.680} \\
\bottomrule
\end{tabular}
\end{table}

The main component ablation under the EDM setting is reported in
Table~\ref{tab:edm_main_ablation}. The full DiFA configuration improves the baseline FID from 5.071 to
2.797, as reported in Table~\ref{tab:edm_main_ablation}. The
corresponding qualitative comparison is shown in
Figure~\ref{fig:dpmdifacifar10}. The identical results of the naive baseline and the w/o History setting indicate that setting $W=1$ disables temporal consensus and collapses the update to the baseline behavior.

The following sensitivity studies examine the robustness of DiFA to
individual hyperparameters under the EDM setting. Each sweep is
designed to vary one hyperparameter while keeping the remaining
settings fixed within that sweep. Since the controlled configurations
may differ across sweeps, the results should be compared primarily
within each table rather than across tables.

\begin{table}[ht]
\centering
\caption{Sensitivity of the history window size $W$ under the EDM setting.}
\label{tab:edm_window}
\begin{tabular}{lcccc}
\toprule
Window Size ($W$) & 1 (Baseline) & 2 & 3 (Ours) & 4 \\
\midrule
FID $\downarrow$ & 5.071 & 3.219 & \textbf{2.797} & 3.791 \\
IS $\uparrow$ & 9.139 & 9.443 & \textbf{9.680} & 9.593 \\
\bottomrule
\end{tabular}
\end{table}

The performance exhibits a U-shaped trend with respect to $W$. A very small window fails to exploit temporal redundancy, while an overly large window introduces stale predictions and causes historical lag. We therefore use $W=3$ as the default setting.

\begin{table}[ht]
\centering
\caption{Sensitivity of the adaptive SNR gating threshold $SNR_{\rm lo}$ under the EDM setting.}
\label{tab:edm_snr_gate}
\begin{tabular}{lcccc}
\toprule
SNR Threshold & -6.0 (Early) & -4.0 & -2.5 & -1.0 (Late) \\
\midrule
FID $\downarrow$ & 3.137 & 3.040 & \textbf{2.986} & 3.068 \\
IS $\uparrow$ & 9.554 & \textbf{9.563} & 9.555 & 9.509 \\
\bottomrule
\end{tabular}
\end{table}

The SNR gating threshold controls when DiFA starts to apply alignment. Early activation may inject corrections during noisy and unstable stages, while overly late activation may miss useful structural formation stages. Within this sweep, moderate activation gives the best FID, and the IS values remain stable across nearby thresholds.

\begin{table}[ht]
\centering
\caption{Sensitivity of the magnitude alignment factor $\phi$ under the EDM setting.}
\label{tab:edm_phi}
\begin{tabular}{lccccc}
\toprule
Rescale Factor ($\phi$) & 0.0 & 0.3 & 0.5 (Ours) & 0.8 & 1.0 \\
\midrule
FID $\downarrow$ & 2.846 & 3.077 & \textbf{2.797} & 3.074 & 3.085 \\
IS $\uparrow$ & 9.661 & 9.555 & \textbf{9.680} & 9.568 & 9.565 \\
\bottomrule
\end{tabular}
\end{table}

The factor $\phi$ stabilizes the magnitude of the deviation-guided prediction. A moderate value preserves residual details while avoiding excessive latent-energy drift.

\begin{table}[ht]
\centering
\caption{Sensitivity of the deviation boosting scale $\gamma_0$ under the EDM setting.}
\label{tab:edm_gamma}
\begin{tabular}{lccccc}
\toprule
Scale ($\gamma_0$) & 1.1 & 1.2 & 1.3 & 1.4 & 1.5 \\
\midrule
FID $\downarrow$ & 4.119 & 3.470 & 3.121 & \textbf{3.071} & 3.285 \\
IS $\uparrow$ & 9.282 & 9.383 & 9.494 & 9.565 & \textbf{9.602} \\
\bottomrule
\end{tabular}
\end{table}

The deviation scale $\gamma_0$ controls the strength of detail re-injection. Too small a value under-corrects the consensus-smoothed prediction, whereas too large a value may amplify unstable residuals. In this sweep, $\gamma_0=1.4$ gives the lowest FID, while $\gamma_0=1.5$ gives the highest IS.

\begin{table}[ht]
\centering
\caption{ Sensitivity of the  compatibility coefficient  $\mu$  under the EDM setting.}
\label{tab:edm_mu}
\begin{tabular}{lccccc}
\toprule
LogSNR Compatibility Coefficient($\mu$) & 0.1 & 0.2 & 0.3 & 0.4 & 0.6 \\
\midrule
FID $\downarrow$ & 3.212 & \textbf{3.205} & 3.220 & 3.256 & 3.373 \\
IS $\uparrow$ & \textbf{9.499} & 9.492 & 9.478 & 9.457 & 9.425 \\
\bottomrule
\end{tabular}
\end{table}

The logSNR compatibility coefficient $\mu$ controls how strongly DiFA penalizes temporally distant or noise-incompatible historical predictions. A small compatibility coefficient is generally stable. In this sweep, $\mu=0.2$ gives the lowest FID, while $\mu=0.1$ gives the highest IS, suggesting that values in the range of $0.1$--$0.3$ provide a reliable balance between using historical evidence and avoiding stale predictions.

\subsection{Cross-Paradigm Validation on Flow-Matching Models}
\label{app:fm_validation}

Although DiFA is introduced in the context of diffusion sampling, its core operation is not tied to a specific diffusion solver. DiFA refines the clean-signal prediction supplied to a downstream update rule. Therefore, it can also be applied to flow-matching models when the predicted velocity or transport state can be converted into an endpoint or clean-signal estimate. This setting provides a useful test of whether the proposed prediction-alignment principle generalizes beyond conventional diffusion samplers.

We evaluate this transferability in two complementary settings. The first is an unguided validation on SiT-XL/2 with ImageNet (256$\times$256), where classifier-free guidance (CFG) is disabled. This experiment was conducted under a lightweight computational setting on 4090 GPUs and was designed as a sanity check for cross-paradigm transfer. As shown in Table~\ref{tab:fm_validation}, DiFA improves the ODE-solver baseline at 5, 10, and 50 NFE without classifier-free guidance. These results indicate that the temporal redundancy exploited by DiFA is also present in flow-matching trajectories.

\begin{table*}[ht]
\centering
\caption{Preliminary no-CFG validation on SiT-XL/2 with ImageNet $256\times256$. This lightweight ODE-solver experiment was conducted on 4090 GPUs to test whether DiFA transfers to flow-matching trajectories without classifier-free guidance.} 
\label{tab:fm_validation}
\begin{tabular}{clccccc}
\toprule
NFE & Method & IS $\uparrow$ & FID $\downarrow$ & sFID $\downarrow$ & Precision $\uparrow$ & Recall $\uparrow$ \\
\midrule
5 & ODE-Solver & 22.5016 & 94.5311 & 60.2238 & 0.2089 & 0.5016 \\
5 & DiFA (Ours) & \textbf{31.8312} & \textbf{72.3969} & \textbf{40.0246} & \textbf{0.2905} & \textbf{0.5694} \\
\midrule
10 & ODE-Solver & 80.2694 & 27.8142 & 13.0268 & 0.5542 & 0.6091 \\
10 & DiFA (Ours) & \textbf{91.7977} & \textbf{21.5990} & \textbf{10.2921} & \textbf{0.5982} & \textbf{0.6256} \\
\midrule
50 & ODE-Solver & 118.6470 & 11.1215 & 7.0554 & 0.6662 & 0.6682 \\
50 & DiFA (Ours) & \textbf{120.2284} & \textbf{10.7515} & \textbf{6.8273} & \textbf{0.6676} & \textbf{0.6694} \\
\bottomrule
\end{tabular}
\end{table*}

The second setting provides a more systematic validation after further adapting DiFA to the flow-matching formulation. In this FM-adapted version, DiFA is applied to the endpoint prediction recovered from the flow-matching transport trajectory, and the temporal consensus is computed in a clean-signal space rather than directly on the velocity field. We also use the flow-matching time variable to define a logSNR-like reliability coordinate, so that the sliding-window consensus and deviation guidance remain compatible with the FM trajectory. This adaptation makes DiFA a clean-prediction refinement module for flow-matching inference rather than a diffusion-specific correction.

Using this FM-adapted implementation, we conduct a CFG-enabled systematic evaluation on H200 GPUs with CFG (=1.5). Table~\ref{tab:difa_sitxl2_imagenet256} expands the preliminary validation into a full solver--NFE--scale matrix under both Euler and Heun2 solvers. This experiment is not intended as a component ablation. Instead, it evaluates cross-paradigm robustness: whether DiFA remains effective under stronger guided generation settings and whether its behavior is stable across solver choices, NFE budgets, and guidance strengths.

\begin{table}[ht]  
\centering
\caption{CFG-enabled systematic validation of the FM-adapted DiFA on SiT-XL/2 with ImageNet $256\times256$. This experiment is conducted on
NVIDIA H200 GPUs with CFG fixed at 1.5. 
We evaluate Euler and Heun2
solvers across multiple NFE budgets and DiFA refinement scales. 
Best
and second-best results among the evaluated DiFA configurations are
shown in bold and underlined, respectively.  
} 
\label{tab:difa_sitxl2_imagenet256}
\begin{tabular}{llccccccc}
\toprule
\textbf{NFE} & \textbf{Method (Scale)} & \textbf{IS $\uparrow$} & \textbf{($\Delta$ IS)} & \textbf{FID $\downarrow$} & \textbf{($\Delta$ FID)} & \textbf{sFID $\downarrow$} & \textbf{Prec. $\uparrow$} & \textbf{Rec. $\uparrow$} \\
\midrule
\multirow{5}{*}{5} & Baseline (Euler) & 58.02 & - & 52.64 & - & 44.10 & 0.3692 & 0.4353 \\
& DiFA (1.25) & 73.90 & (+27.4\%) & 40.81 & (-22.5\%) & 33.33 & 0.4407 & 0.4290 \\
& DiFA (1.5) & 86.49 & (+49.1\%) & 32.86 & (-37.6\%) & 26.11 & 0.4969 & 0.4322 \\
& DiFA (1.7) & \underline{94.15} & \underline{(+62.3\%)} & \underline{28.52} & \underline{(-45.8\%)} & \underline{22.23} & \underline{0.5318} & \underline{0.4383} \\
& DiFA (1.75) & \textbf{95.88} & \textbf{(+65.2\%)} & \textbf{27.61} & \textbf{(-47.5\%)} & \textbf{21.46} & \textbf{0.5400} & \textbf{0.4406} \\
\cmidrule{1-9}
\multirow{5}{*}{10} & Baseline (Euler) & 184.11 & - & 8.35 & - & 8.83 & 0.7303 & 0.5234 \\
& DiFA (1.25) & 200.94 & (+9.1\%) & 6.08 & (-27.2\%) & 6.40 & 0.7616 & 0.5311 \\
& DiFA (1.5) & 213.98 & (+16.2\%) & 4.80 & (-42.5\%) & \underline{5.44} & 0.7790 & 0.5342 \\
& DiFA (1.7) & \underline{219.79} & \underline{(+19.4\%)} & \underline{4.28} & \underline{(-48.7\%)} & \textbf{5.40} & \underline{0.7884} & \underline{0.5361} \\
& DiFA (1.75) & \textbf{220.91} & \textbf{(+20.0\%)} & \textbf{4.20} & \textbf{(-49.7\%)} & 5.47 & \textbf{0.7912} & \textbf{0.5378} \\
\cmidrule{1-9}
\multirow{5}{*}{20} & Baseline (Euler) & 231.60 & - & 3.33 & - & 5.33 & 0.7902 & 0.5736 \\
& DiFA (1.25) & 241.02 & (+4.1\%) & 2.76 & (-17.2\%) & 4.45 & 0.7997 & 0.5770 \\
& DiFA (1.5) & 248.54 & (+7.3\%) & 2.48 & (-25.5\%) & \textbf{4.24} & 0.8060 & \underline{0.5784} \\
& DiFA (1.7) & \underline{251.81} & \underline{(+8.7\%)} & \underline{2.42} & \underline{(-27.5\%)} & \underline{4.42} & \underline{0.8105} & 0.5778 \\
& DiFA (1.75) & \textbf{252.71} & \textbf{(+9.1\%)} & \textbf{2.41} & \textbf{(-27.5\%)} & 4.51 & \textbf{0.8121} & \textbf{0.5804} \\
\midrule
 \multirow{5}{*}{5} & Baseline (Heun2) & 148.39 & - & 14.64 & - & 22.11 & 0.6590 & 0.5155 \\
& DiFA (1.25) & 162.01 & (+9.2\%) & 11.67 & (-20.3\%) & 17.88 & 0.6827 & 0.5294 \\
& DiFA (1.5) & 174.30 & (+17.5\%) & 9.53 & (-34.9\%) & 14.78 & 0.7049 & 0.5388 \\
& DiFA (1.7) & \underline{181.66} & \underline{(+22.4\%)} & \underline{8.26} & \underline{(-43.6\%)} & \underline{12.92} & \underline{0.7158} & \underline{0.5538} \\
& DiFA (1.75) & \textbf{183.22} & \textbf{(+23.5\%)} & \textbf{7.99} & \textbf{(-45.4\%)} & \textbf{12.53} & \textbf{0.7178} & \textbf{0.5560} \\
\cmidrule{1-9}
 \multirow{5}{*}{10} & Baseline (Heun2) & 233.24 & - & 3.15 & - & 6.59 & 0.7882 & 0.5852 \\
& DiFA (1.25) & 241.89 & (+3.7\%) & 2.66 & (-15.7\%) & 5.65 & 0.7959 & 0.5885 \\
& DiFA (1.5) & 247.92 & (+6.3\%) & 2.37 & (-24.9\%) & 5.09 & 0.7989 & 0.5894 \\
& DiFA (1.7) & \underline{253.58} & \underline{(+8.7\%)} & \underline{2.22} & \underline{(-29.6\%)} & \underline{4.86} & \underline{0.8034} & \textbf{0.5934} \\
& DiFA (1.75) & \textbf{254.45} & \textbf{(+9.1\%)} & \textbf{2.19} & \textbf{(-30.3\%)} & \textbf{4.82} & \textbf{0.8046} & \underline{0.5925} \\
\cmidrule{1-9}
 \multirow{5}{*}{20} & Baseline (Heun2) & 250.65 & - & 2.20 & - & 4.73 & 0.8009 & 0.6000 \\
& DiFA (1.25) & 254.11 & (+1.4\%) & 2.08 & (-5.6\%) & 4.52 & 0.8039 & \textbf{0.5987} \\
& DiFA (1.5) & 259.56 & (+3.6\%) & 2.02 & (-8.4\%) & \textbf{4.47} & 0.8076 & \underline{0.5986} \\
& DiFA (1.7) & \underline{261.63} & \underline{(+4.4\%)} & \textbf{2.01} & \textbf{(-8.6\%)} & \underline{4.51} & \underline{0.8103} & 0.5964 \\
& DiFA (1.75) & \textbf{262.39} & \textbf{(+4.7\%)} & \underline{2.01} & \underline{(-8.6\%)} & 4.53 & \textbf{0.8117} & 0.5967 \\
\bottomrule
\end{tabular}
\end{table}

The two settings should therefore be interpreted as complementary evidence rather than directly comparable benchmarks. The preliminary no-CFG experiment verifies the basic transferability of DiFA to flow-matching models. The CFG-enabled H200 experiment, obtained after the FM-specific adaptation, provides a stronger and more realistic validation under guided sampling and multiple solver configurations.

The results in Table~\ref{tab:difa_sitxl2_imagenet256} show that DiFA
consistently improves FID, IS, sFID, and precision over the
corresponding flow-matching baselines. Recall remains comparable in
most settings, although a small decrease is observed for Heun2 at
20 NFE, indicating a mild precision--recall trade-off in this regime.  
The gains are especially pronounced in the low-NFE regime, where trajectory discretization error and clean-signal prediction uncertainty are more severe. For example, in the CFG-enabled setting, DiFA reduces the Euler 5-NFE FID from (52.64) to (27.61), and reduces the Heun2 5-NFE FID from (14.64) to (7.99). At 10 NFE, DiFA still provides substantial improvements for both Euler and Heun2. At 20 NFE, where the baseline trajectories are already stronger, the relative margin naturally becomes smaller but remains generally positive.

The DiFA refinement-scale sweep further shows that the improvement is not caused by a single isolated hyperparameter choice. Scales around (1.7)--(1.75) usually provide the strongest FID and IS performance, while smaller scales still improve over the corresponding baselines. These results support the view that DiFA acts as a solver-compatible inference-time refinement mechanism and that its temporal prediction-alignment principle can extend from diffusion samplers to flow-matching trajectories.  \nocite{CHEN2026112442, wu2023fast, xu2024sparse,li2024coupled,chen2024rethinking,chen2025dequantified,li2026loramixer}

The quantitative improvements are also reflected in the qualitative
comparison shown in Figure~\ref{fig:guidedimagenet5nfe}. Under the
5-NFE setting, DiFA produces samples with better structural
coherence and fewer truncation artifacts than the baseline solver.
The comparison provides qualitative evidence that the numerical gains
are accompanied by improved visual quality.

\begin{figure}[t]
    \centering
    \begin{subfigure}[b]{0.49\linewidth}
        \centering
        \centerline{  SiT with Baseline Solver }
        \vspace{4pt} 
        \includegraphics[width=\linewidth]{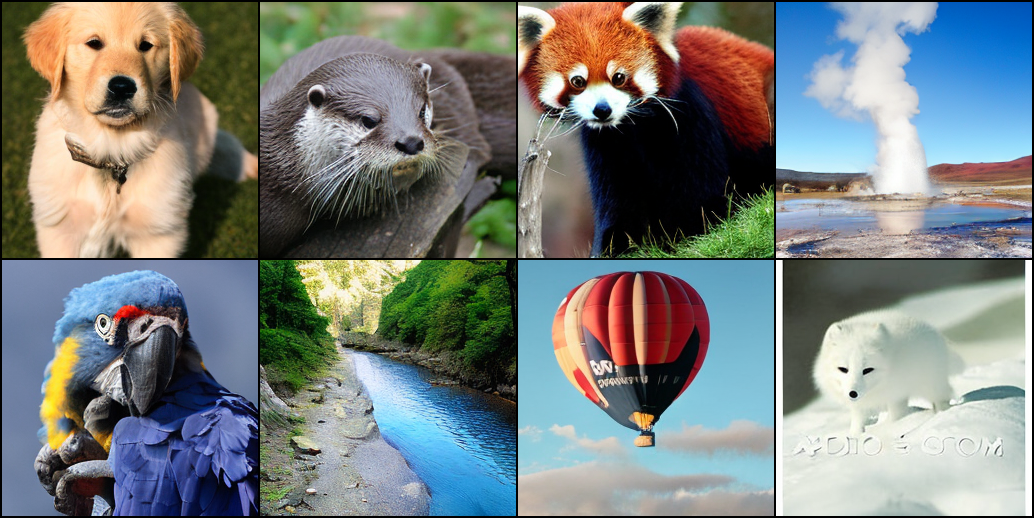}
    \end{subfigure}
    \hfill 
    \begin{subfigure}[b]{0.49\linewidth}
        \centering
        \centerline{ SiT with DiFA Refined Prediction}
        \vspace{4pt}
        \includegraphics[width=\linewidth]{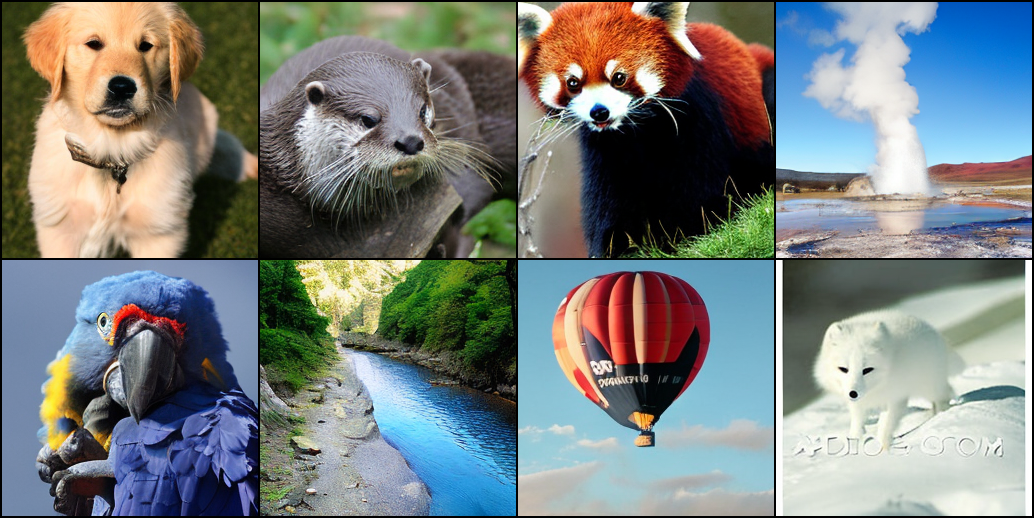}
    \end{subfigure}  
    \caption{Qualitative comparison on ImageNet
($256\times256$) using the pretrained SiT-XL/2 model with the Euler
solver, 5 function evaluations ($\text{NFE} = 5$), and a classifier-free guidance (CFG)
scale of 4.0. Compared with the baseline, DiFA
preserves better structural coherence and local details under this
aggressive few-step setting.}
    \label{fig:guidedimagenet5nfe}
    \vspace{-0.3cm}
\end{figure}

\end{document}